\documentclass[10pt,twocolumn,letterpaper]{article}

\usepackage[pagenumbers]{cvpr}  

\usepackage{graphicx}
\usepackage{bm}
\usepackage{amsmath}
\usepackage{amssymb}
\usepackage{booktabs}
\usepackage{multirow}
\usepackage{amsthm}

\usepackage{caption}
\usepackage{subcaption}
\usepackage{booktabs}
\usepackage{tabularx}

\usepackage{adjustbox}
\usepackage{threeparttable}
\usepackage{makecell}

\usepackage{wrapfig}

\usepackage{algorithm}
\usepackage{algorithmicx}
\usepackage[noend]{algpseudocode}
\usepackage{bbding}

\usepackage{mathtools}
\usepackage{thmtools}
\usepackage{thm-restate}

\usepackage{microtype}      
\usepackage{xcolor}

\usepackage[accsupp]{axessibility}  

\newtheorem{definition}{Definition}
\newtheorem{lemma}{Lemma}

\newtheorem{proposition}{Proposition}
\newtheorem{assumption}{Assumption}

\newcommand{\xB}{\mathbf{x}}
\newcommand{\yB}{\mathbf{y}}
\newcommand{\zB}{\mathbf{z}}
\newcommand{\Acal}{\mathcal{A}}
\newcommand{\Ecal}{\mathcal{E}}
\newcommand{\Fcal}{\mathcal{F}}
\newcommand{\Scal}{\mathcal{S}}
\newcommand{\Tcal}{\mathcal{T}}
\newcommand{\Xcal}{\mathcal{X}}
\newcommand{\Ycal}{\mathcal{Y}}
\newcommand{\Zcal}{\mathcal{Z}}
\newcommand{\Ebb}{\mathbb{E}}
\newcommand{\given}{\,\vert\,}
\newcommand{\definedas}{\vcentcolon=}

\makeatletter
\@namedef{ver@everyshi.sty}{}
\makeatother
\usepackage{tikz}
\usetikzlibrary{arrows.meta, positioning}

\usepackage[pagebackref,breaklinks,colorlinks]{hyperref}

\usepackage[capitalize]{cleveref}
\Crefname{section}{Sec.}{Secs.}
\Crefname{section}{Section}{Sections}
\Crefname{table}{Table}{Tables}
\Crefname{table}{Tab.}{Tabs.}

\def\cvprPaperID{7507} 
\def\confName{CVPR}
\def\confYear{2022}
\newcommand{\name}{OoD-Bench}

\usepackage{lipsum}

\begin{document}

\title{\name: Quantifying and Understanding Two Dimensions of Out-of-Distribution Generalization}
\author{Nanyang Ye$^{1}$\thanks{Nanyang Ye and Kaican Li are the joint first authors. Nanyang Ye is the corresponding author.} , Kaican Li$^{2}$\footnotemark[1] , Haoyue Bai$^{3}$\thanks{Work is done during Haoyue and Runpeng's internships at Shanghai Jiao Tong University.} , Runpeng Yu$^{4}$\footnotemark[2] , Lanqing Hong$^{2}$, Fengwei Zhou$^{2}$, \\ Zhenguo Li$^{2}$, Jun Zhu$^{4}$\\
$^{1}$ Shanghai Jiao Tong University \quad $^{2}$ Huawei Noah's Ark Lab\\
$^{3}$ Hong Kong University of Science and Technology \quad $^{4}$ Tsinghua University \\
{\tt\small ynylincoln@sjtu.edu.cn,  mjust.lkc@gmail.com,  hbaiaa@cse.ust.hk, yrp19@mails.tsinghua.edu.cn,}\\
{\tt\small \{honglanqing, zhoufengwei, li.zhenguo\}@huawei.com, dcszj@mail.tsinghua.edu.cn}
}

\maketitle

\begin{abstract}
Deep learning has achieved tremendous success with independent and identically distributed (i.i.d.) data. However, the performance of neural networks often degenerates drastically when encountering out-of-distribution (OoD) data, \ie, when training and test data are sampled from different distributions.
While a plethora of algorithms have been proposed for OoD generalization, our understanding of the data used to train and evaluate these algorithms remains stagnant.
In this work, we first identify and measure two distinct kinds of distribution shifts that are ubiquitous in various datasets.
Next, through extensive experiments, we compare OoD generalization algorithms across two groups of benchmarks, each dominated by one of the distribution shifts, revealing their strengths on one shift as well as limitations on the other shift. Overall, we position existing datasets and algorithms from different research areas seemingly unconnected into the same coherent picture. It may serve as a foothold that can be resorted to by future OoD generalization research.
Our code is available at \url{https://github.com/ynysjtu/ood_bench}.
\end{abstract}

\section{Introduction}
Deep learning has been widely adopted in various applications of computer vision~\cite{he2016deep} and natural language processing~\cite{devlin2018bert} with great success, under the implicit assumption that the training and test data are drawn from the same distribution, which is known as the independent and identically distributed (i.i.d.) assumption.
While neural networks often exhibit super-human generalization performance on the training distribution, they can be susceptible to minute changes in the test distribution~\cite{recht2019imagenet, szegedy2014intriguing}.
This is problematic because sometimes true underlying data distributions are significantly underrepresented or misrepresented by the limited training data at hand.
In the real world, such mismatches are commonly observed~\cite{koh2020wilds, geirhos2020shortcut}, and have led to significant performance drops in many deep learning algorithms~\cite{bahng2019rebias, krueger2020outofdistribution, mancini2020towards}.
As a result, the reliability of current learning systems is substantially undermined in critical applications such as medical imaging~\cite{castro2020causality, albadawy2018deep}, autonomous driving~\cite{dai2018dark, volk2019towards, alcorn2019strike, sato2020security, michaelis2019benchmarking}, and security systems~\cite{huang2020survey}.

\emph{Out-of-Distribution (OoD) Generalization}, the task of generalizing under such distribution shifts, has been fragmentarily researched in different areas, such as Domain Generalization (DG)~\cite{blanchard2011generalizing, muandet2013domain, wang2021generalizing, zhou2021domain}, Causal Inference~\cite{pearl2010causal, peters2017elements}, and Stable Learning~\cite{zhang2021deep}.
In the setting of OoD generalization, models usually have access to multiple training datasets of the same task collected in different environments.
The goal of OoD generalization algorithms is to learn from these different but related training environments and then extrapolate to unseen test environments~\cite{arjovsky2020out, shen2021towards}.
Driven by this motivation, numerous algorithms have been proposed over the years~\cite{zhou2021domain}, each claimed to have surpassed all its precedents on a particular genre of benchmarks.
However, a recent work~\cite{gulrajani2021in} suggests that the progress made by these algorithms might have been overestimated---most of the advanced learning algorithms tailor-made for OoD generalization are still on par with the classic Empirical Risk Minimization~(ERM)~\cite{vapnik1998statistical}.

\begin{figure*}[t]
    \centering
    \includegraphics[width=\linewidth]{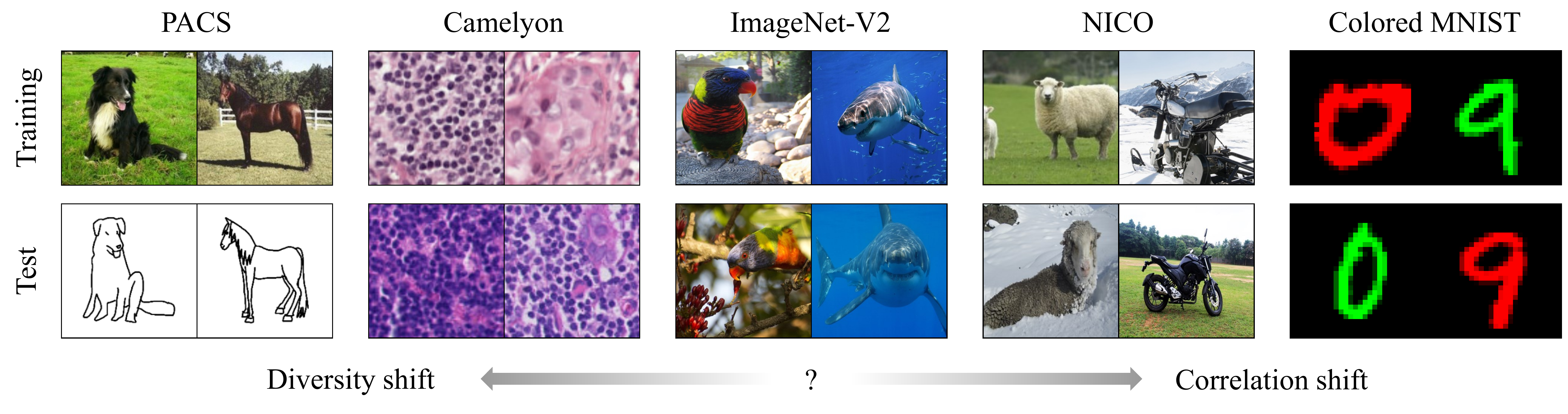}
    \caption{Examples of image classification datasets demonstrating different kinds of distribution shifts.
    While it is clear that the datasets at both ends exhibit apparent distribution shifts, in the middle, it is hard to distinguish the differences in distribution between the training dataset and the test dataset (\eg, ImageNet~\cite{deng2009imagenet} and ImageNet-V2~\cite{recht2019imagenet}), which represent a large body of realistic OoD datasets. This motivates us to quantify the distribution shifts in these OoD datasets.}
    \label{fig:samples}
\end{figure*}

\textbf{In this work, we provide a quantification for the distribution shift exhibited in OoD datasets from different research areas and evaluate the effectiveness of OoD generalization algorithms on these datasets, revealing a possible reason as to why these algorithms appear to be no much better than ERM, which is left unexplained in previous work~\cite{gulrajani2021in}.}
We find that incumbent datasets exhibiting distribution shifts can be generally divided into two categories of different characteristics, whereas the majority of the algorithms are only able to surpass ERM in at most one of the categories.
We hypothesize that the phenomenon is due to the influence of two distinct kinds of distribution shift, namely \emph{diversity shift} and \emph{correlation shift}, while preexisting literature often focuses on merely one of them.
The delineation of diversity and correlation shift provides us with a unified picture for understanding distribution shifts.
Based on the findings and analysis in this work, we make three recommendations for future OoD generalization research:

\begin{itemize}
    \item Evaluate OoD generalization algorithms comprehensively on two types of datasets, one dominated by diversity shift and the other dominated by correlation shift.
    We provide a method to estimate the strength of these two distribution shifts on any labeled dataset.
    \item \noindent Investigate the nature of distribution shift in OoD problems before designing algorithms since the optimal treatment for different kinds of distribution shift may be different.
    \item Design large-scale datasets that more subtly capture real-world distribution shifts as imperceptible distribution shifts can also be conspicuous to neural networks.
\end{itemize}

\section{Diversity Shift and Correlation Shift}
\label{sec:definition}

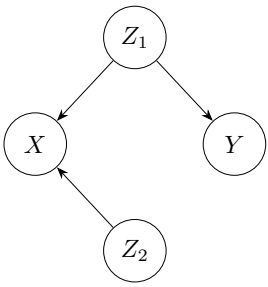
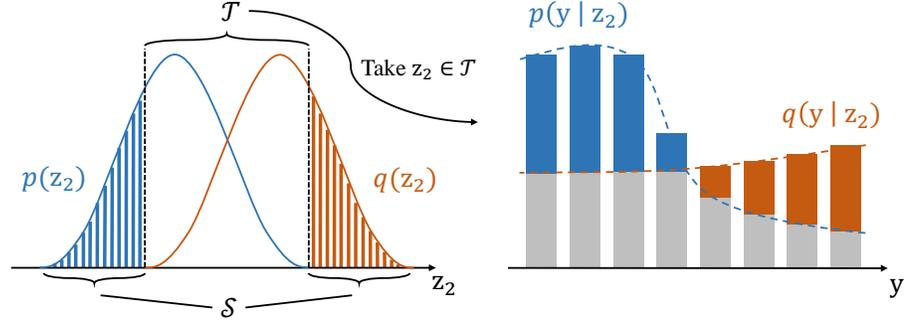
\begin{figure*}[t]
    \centering
    \begin{subfigure}{0.25\textwidth}
        \centering
        \begin{adjustbox}{max width=0.8\textwidth}
            \begin{tikzpicture}
                \begin{scope}[every node/.style={circle, draw, minimum size=2.5em}]
                    \node (A) at (-1.4, 0) {$X$};
                    \node (B) at ( 1.4, 0) {$Y$};
                    \node (C) at ( 0, 1.5) {$Z_1$};
                    \node (D) at ( 0,-1.5) {$Z_2$};
                \end{scope}
                \begin{scope}[>={Stealth[black]},
                          every edge/.style={draw=black}]
                    \path [->] (C) edge node {} (A);
                    \path [->] (D) edge node {} (A);
                    \path [->] (C) edge node {} (B);
                \end{scope}
            \end{tikzpicture}
        \end{adjustbox}
        \vspace{17pt}
        \caption{Causal graph depicting the causal influence among the concerned variables.}
        \label{fig:causal-graph}
    \end{subfigure}
    \hfill
    \begin{subfigure}{0.7\textwidth}
        \centering
        \includegraphics[width=\textwidth]{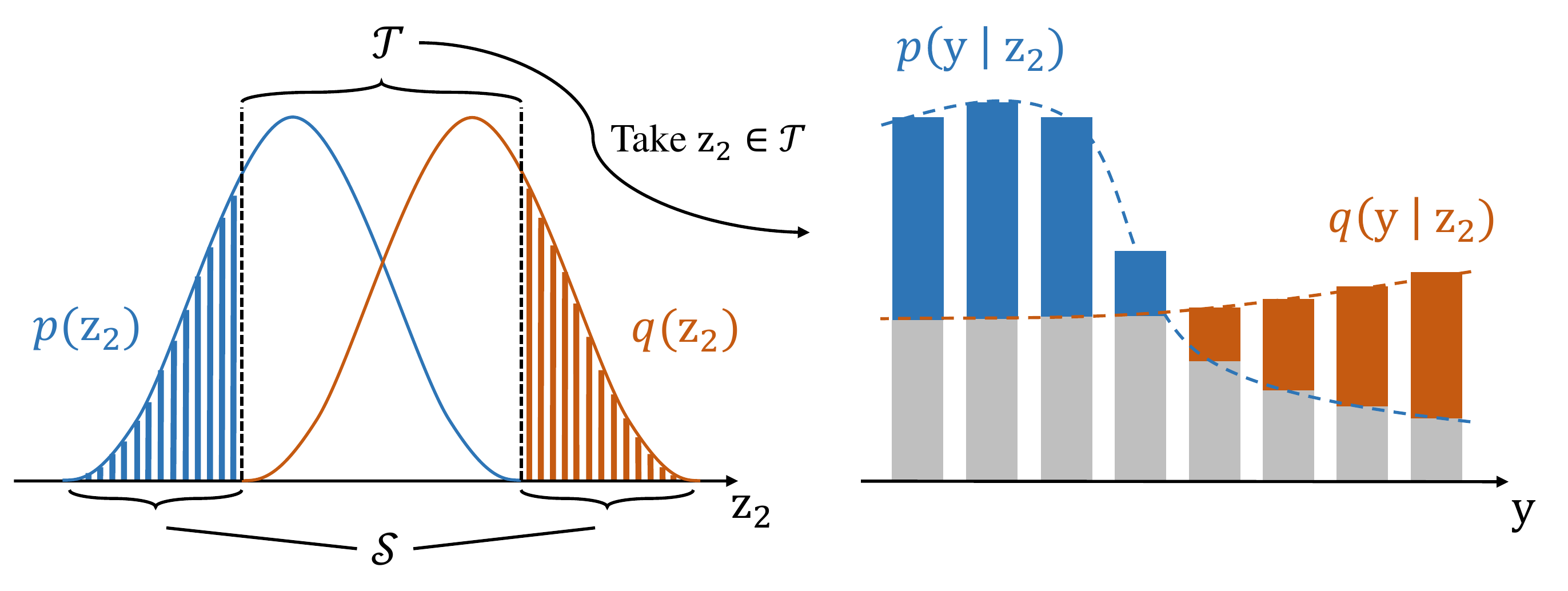}
        \caption{Illustration of diversity shift and correlation shift. Diversity shift amounts to half of the summed area of the colored regions in the left figure. Correlation shift is an integral over $\Tcal$, where every integrand can be seen as the summed heights of the colored bars in the right figure then weighted by the square root of $p(\zB_2) \cdot q(\zB_2)$.}
        \label{fig:div-cor-illus}
    \end{subfigure}
    \caption{Explanatory illustrations for diversity and correlation shift.  Diversity shift is defined by the support set's difference of the latent environment's distribution while correlation shift is defined by the probability density function's difference on the same support set.}
\end{figure*}

Normally, datasets such as VLCS~\cite{torralba2011unbiased} and PACS~\cite{li2017deeper} that consist of multiple domains are used to train and evaluate DG models.
In these datasets, each domain represents a certain spectrum of diversity in data.\footnote{An implicit assumption of DG over the years is that each domain is \emph{distinct} from one another, which is a major difference between DG and OoD generalization as the latter considers a more general setting.}
During experiments, these domains are further grouped into training and test domains, leading to the diversity shift.
Although extensive research efforts have been dedicated to the datasets dominated by diversity shift, it is not until recently that \cite{arjovsky2019invariant} draws attention to another challenging generalization problem stemmed from spurious correlations. Colored MNIST, a variant of MNIST~\cite{lecun1998gradient}, is constructed by coloring the digits with either red or green to highlight the problem. The colored digits are arranged into training and test environments such that the labels and colors are strongly correlated, but the correlation flips across the environments, creating correlation shift.

As shown in \Cref{fig:samples}, diversity and correlation shift are of clearly different nature. At the extremes, discrepancies between training and test environments become so apparent, causing great troubles for algorithms trying to generalize~\cite{li2017deeper, arjovsky2019invariant}. Interestingly, in some real-world cases such as ImageNet versus ImageNet-V2 where the discrepancy is virtually imperceptible, neural networks are still unable to generalize satisfactorily, of which the reason is not fully understood~\cite{recht2019imagenet}. In \Cref{fig:estimation}, our estimation of diversity and correlation shift shed some light on the issue---there is a non-trivial degree of correlation shift between the original ImageNet and the variant. Besides, other OoD datasets have also shown varying degrees of diversity and correlation shift.

\paragraph{Formal Analysis.}
In the setting of supervised learning, every input $\xB \in \Xcal$ is assigned with a label $\yB \in \Ycal$ by some fixed labeling rule $f: \Xcal \to \Ycal$.
The inner mechanism of $f$ usually depends on a particular set of features $\Zcal_1$, whereas the rest of the features $\Zcal_2$ are not causal to prediction. For example, we assign the label ``airplane'' to the image of an airplane regardless of its color or whether it is landed or flying.
The causal graph in \Cref{fig:causal-graph} depicts the interplay among the underlying random variables of our model: the input variable $X$ is determined by the latent variables $Z_1$ and $Z_2$, whereas the target variable $Y$ is determined by $Z_1$ alone. Similar graphs can be found in \cite{ahuja2021empirical,liu2021learning,mitrovic2021representation,ouyang2021causality}.

Given a labeled dataset, consider its training environment $\Ecal_{\textnormal{tr}}$ and test environment $\Ecal_{\textnormal{te}}$ as distributions with probability functions $p$ and $q$ respectively.
For ease of exposition, we assume no label shift \cite{azizzadenesheli2018regularized} across the environments, \ie $p(\yB) = q(\yB)$ for every $\yB \in \Ycal$.\footnote{As a side note, by this assumption we do \emph{not} ignore the existence of label shift in datasets. In practice, datasets with label shift can be made to satisfy this assumption by techniques such as sample reweighting.}
Without loss of generality, we further assume that $\Ecal_{\textnormal{tr}}$ and $\Ecal_{\textnormal{te}}$ share the same labeling rule $f$, complementing the causal graph. To put it in the language of \emph{causality} \cite{pearl2000causality}, it means that the \emph{direct cause} of $Y$ (which is $Z_1$) is observable in both environments and the causal mechanism that $Z_1$ exerts on $Y$ is stable at all times.
Formally, it dictates the following property for every $\zB \in \Zcal_1$:%
\begin{equation}
    p(\zB) \cdot q(\zB) \neq 0 \:\land\: \forall\, \yB \in \Ycal: p(\yB \given \zB) = q(\yB \given \zB).
\end{equation}
\noindent The existence of such invariant features makes OoD generalization possible. On the other hand, the presence of $\zB \in \Zcal_2$ possessing the opposite property,
\begin{equation}
    \label{eq:pqz2}
    p(\zB) \cdot q(\zB) = 0 \:\lor\: \exists\, \yB \in \Ycal: p(\yB \given \zB) \neq q(\yB \given \zB),
\end{equation}
makes OoD generalization challenging.
From \eqref{eq:pqz2}, we can see that $\Zcal_2$ consists of two kinds of features.
\textit{Intuitively, diversity shift stems from the first kind of features in $\Zcal_2$ since the diversity of data is embodied by novel features not shared by the environments; whereas correlation shift is caused by the second kind of  features in $\Zcal_2$ which is spuriously correlated with some $\yB$.}
Based on this intuition, we partition $\Zcal_2$ into two subsets,
\begin{equation}
    \label{eq:ST}
    \begin{split}
        \Scal &\definedas \{ \zB \in \Zcal_2 \mid p(\zB) \cdot q(\zB) = 0 \}, \\
        \Tcal &\definedas \{ \zB \in \Zcal_2 \mid p(\zB) \cdot q(\zB) \neq 0 \},
    \end{split}
\end{equation}
that are respectively responsible for diversity shift and correlation shift between the environments.
We then define the quantification formula of the two shifts as follows:
\begin{definition}[\textbf{Diversity Shift and Correlation Shift}]
    Given $\Scal$ and $\Tcal$ defined in \eqref{eq:ST}, the proposed quantification formula of diversity shift and correlation shift between two data distributions $p$ and $q$ is given by
    \label{def:div-cor}
    \begin{equation*}
        \begin{split}
            D_\textnormal{div}(p, q) &\definedas \frac{1}{2} \int_{\Scal} |\:\! p(\zB) - q(\zB) |\, d \zB, \\
            D_\textnormal{cor}(p, q) &\definedas \frac{1}{2} \int_{\Tcal} \sqrt{p(\zB) \cdot q(\zB)} \sum_{\yB \in \Ycal} {\big|\:\! p(\yB \given \zB) - q(\yB \given \zB) \big|}\, d \zB,
        \end{split}
    \end{equation*}
    where we assume $\Ycal$ to be discrete.
\end{definition}
\noindent \Cref{fig:div-cor-illus} illustrates the above definition when $\zB$ is unidimensional.
It can be proved that $D_\textnormal{div}$ and $D_\textnormal{cor}$ are always bounded within $[0, 1]$ (see \textbf{Proposition~\ref{prop:zero-one-bound}} in Appendix~\ref{app:proofs}).
In particular, the square root in the formulation of correlation shift serves as a coefficient regulating the integrand because features that hardly appear in either environment should have a small contribution to the correlation shift overall.
Nevertheless, we are aware that these are not the only viable formulations, yet they produce intuitively reasonable and numerically stable results even when estimated by a simple method described next.

\paragraph{Practical estimation.}

\begin{figure}
\centering
\includegraphics[width=\linewidth]{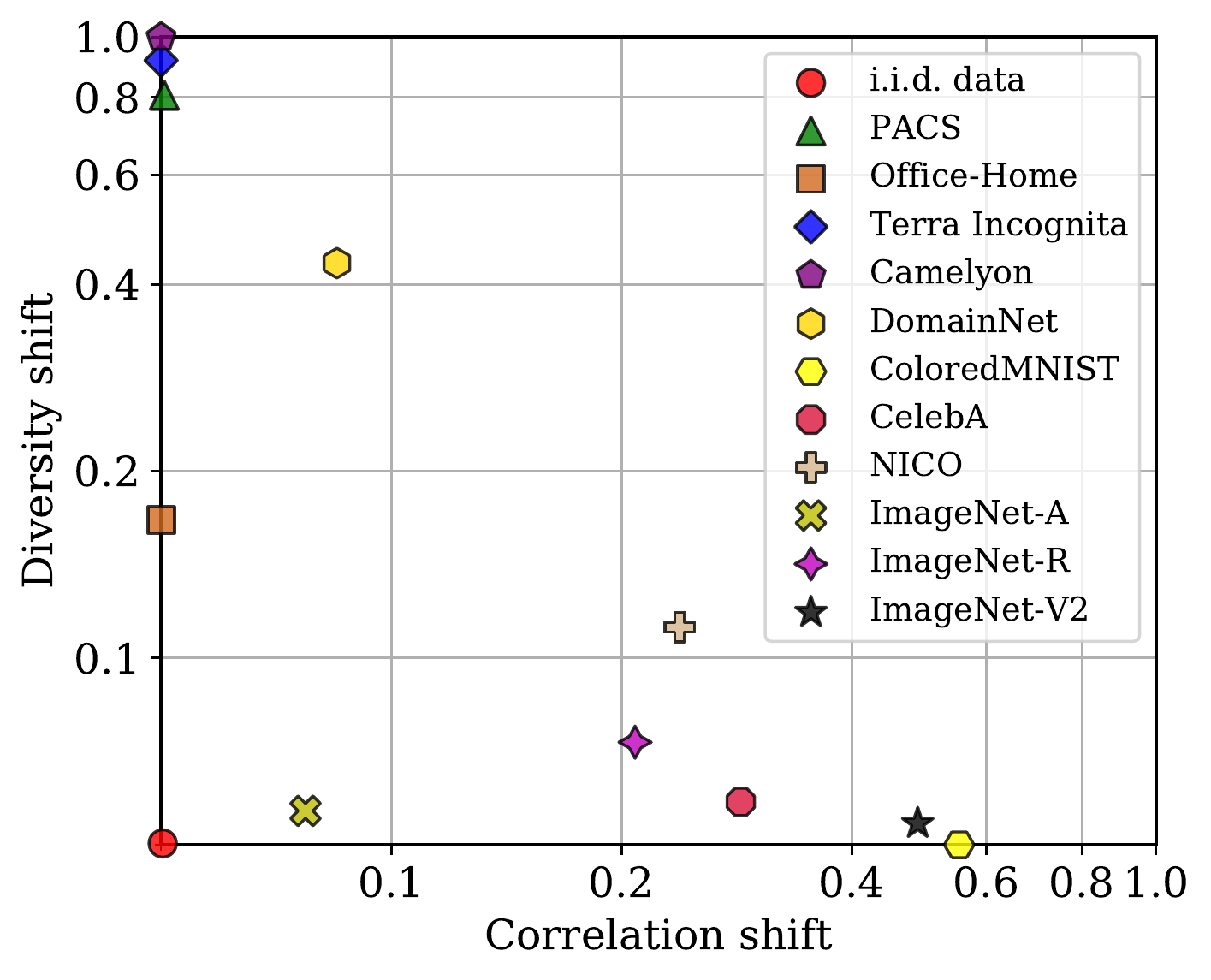}
\caption{Estimation of diversity and correlation shift in various datasets. For ImageNet variants, the estimates are computed with respect to the original ImageNet. See Appendix~\ref{app:estimation-results} for the results in numeric form with error bars.}
\label{fig:estimation}
\end{figure}

Given a dataset sampled from $\Ecal_\textnormal{tr}$ and another dataset (of equal size) sampled from $\Ecal_\textnormal{te}$, a neural network is first trained to discriminate the environments.
The network consists of a feature extractor $g: \Xcal \to \Fcal$ and a classifier $h: \Fcal \times \Ycal \to [0, 1]$, where $\Fcal$ is some learned representation of $\Xcal$.
The mapping induces two joint distributions over $\Xcal \times \Ycal \times \Fcal$, one for each environment, with probability functions denoted by $\Hat{p}$ and $\Hat{q}$.
For every example from either $\Ecal_\textnormal{tr}$ or $\Ecal_\textnormal{te}$, the network tries to tell which environment the example is actually sampled from, in order to minimize the following objective:
\begin{equation}
    \label{eq:objective}
    \mathbb{E}_{(\xB, \yB) \sim \Ecal_\textnormal{tr}}\ell(\Hat{e}_{\xB, \yB}, 0) \,+\, \mathbb{E}_{(\xB, \yB) \sim \Ecal_\textnormal{te}}\ell(\Hat{e}_{\xB, \yB}, 1),
\end{equation}
where $\Hat{e}_{\xB, \yB} = h(g(\xB), \yB)$ is the predicted environment and $\ell$ is some loss function.
The objective forces $g$ to extract those features whose joint distribution with $Y$ varies across the environments so that $h$ could make reasonably accurate predictions. This is formalized by the theorem below.
\begin{restatable}[]{theorem}{feature}
    \label{thm:feature-space}
    The classification accuracy of a network trained to discriminate two environments is bounded above by $\frac{1}{2} \int_{\Xcal} \max \{p(\xB), q(\xB)\}$, as the data size tends to infinity. This optimal performance is attained only when the following condition holds: for every $\xB \in \Xcal$ that is \emph{not} i.i.d\onedot in the two environments, \ie $p(\xB) \neq q(\xB)$, there exists some $\yB \in \Ycal$ such that $\Hat{p}(\yB, \zB) \neq \Hat{q}(\yB, \zB)$ where $\zB = g(\xB)$.
\end{restatable}
The proof is provided in \cref{app:proofs}.
After obtaining the features $\Fcal$ extracted by $g$, we use Kernel Density Estimation~(KDE)~\cite{rosenblatt1956remarks, parzen1962onestimation} to estimate $\Hat{p}$ and $\Hat{q}$ over $\Fcal$.
Subsequently, $\Fcal$ is partitioned by whether $\Hat{p}(\zB) \cdot \Hat{q}(\zB)$ is close to zero, in correspondence to $\Scal$ and $\Tcal$, into two sets of features that are responsible for diversity and correlation shift respectively.
The integrals in Definition~\ref{def:div-cor} are then approximated by Monte Carlo Integration under importance sampling \cite{newman1999monte}.
A caveat in evaluating the term $|p(\yB \given \zB) - q(\yB \given \zB)|$ in $D_\textnormal{cor}(p, q)$ is that the conditional probabilities are computationally intractable for $\zB$ is continuous.
Instead, the term is computed by the following equivalent formula as an application of Bayes' theorem:
\begin{equation}
    \left|\frac{\Hat{p}(\yB) \cdot \Hat{p}(\zB \given \yB)}{\Hat{p}(\zB)} - \frac{\Hat{q}(\yB) \cdot \Hat{q}(\zB \given \yB)}{\Hat{q}(\zB)}\right|,
\end{equation}
where $\Hat{p}(\zB \given \yB)$ and $\Hat{q}(\zB \given \yB)$ can be approximated individually for every $\yB \in \Ycal$ again by KDE.
See Appendix~\ref{app:estimation-method} for more details of our method including pseudo codes of the whole procedure.

We have also shown that in theory the extracted features would converge to a unique solution as the network width grows to infinity using Neural Tangent Kernel \cite{jacot2018neural}.
It suggests that as long as the network has sufficient capacity, we can always obtain similar results within a small error bound.
To empirically verify this, we have also experimented with different network architectures which demonstrates the stability of our estimation (see Appendix~\ref{app:arch}).

The results in \Cref{fig:estimation} are obtained by the aforementioned method.
Most of the existing OoD datasets lie over or near the axes, dominated by one kind of shift.
For datasets under unknown distribution shift such as ImageNet-A~\cite{hendrycks2019natural}, ImageNet-R~\cite{hendrycks2020many}, and ImageNet-V2, our method successfully decomposes the shift into the two dimensions of diversity and correlation, and therefore one may choose the appropriate algorithms based on the estimation.
As shown by our benchmark results in the next section, such choices might be crucial as most OoD generalization algorithms do not perform equally well on two groups of datasets, one dominated by diversity shift and the other dominated by correlation shift.

\section{Experiment}
\label{sec:experiment}
Previously, we have numerically positioned OoD datasets in the two dimensions of distribution shift. In this section, we run algorithms on these datasets to reveal the two-dimensional trend for existing datasets and algorithms. All experiments are conducted on Pytorch 1.4 with Tesla V100 GPUs.
Our code for the following benchmark experiments is modified from the DomainBed~\cite{gulrajani2021in} code suite.

\subsection{Benchmark}
\label{sec:benchmark}

\begin{table*}[t]
    \centering
	\begin{threeparttable}
        \begin{tabularx}{\linewidth}{l*{4}{X}rr}
            \toprule
            \textbf{Algorithm} & \textbf{PACS} & \textbf{OfficeHome} & \textbf{TerraInc} & \textbf{Camelyon17} & \textbf{Average} & \textbf{Ranking score}\\
            \midrule
            RSC~\cite{huang2020self} &
                82.8 $\pm$ 0.4\tnote{$\uparrow$} &
                62.9 $\pm$ 0.4\tnote{$\downarrow$} &
                43.6 $\pm$ 0.5\tnote{$\uparrow$} &
                94.9 $\pm$ 0.2\tnote{$\uparrow$} &
                71.1 &
                \texttt{+2} \\
            MMD~\cite{li2018domain} & 
                81.7 $\pm$ 0.2\tnote{$\uparrow$} &
                63.8 $\pm$ 0.1\tnote{$\uparrow$} &
                38.3 $\pm$ 0.4\tnote{$\downarrow$} &
                94.9 $\pm$ 0.4\tnote{$\uparrow$} &
                69.7 &
                \texttt{+2} \\
            SagNet~\cite{nam2019reducing} &
                81.6 $\pm$ 0.4\tnote{$\uparrow$} &
                62.7 $\pm$ 0.4\tnote{$\downarrow$} &
                42.3 $\pm$ 0.7\hphantom{\tnote{$\downarrow$}} &
                95.0 $\pm$ 0.2\tnote{$\uparrow$} &
                70.4 &
                \texttt{+1} \\
            \textbf{ERM}\vphantom{\tnote{$\downarrow$}}~\cite{vapnik1998statistical} &
                81.5 $\pm$ 0.0\hphantom{\tnote{$\downarrow$}} &
                63.3 $\pm$ 0.2\hphantom{\tnote{$\downarrow$}} &
                42.6 $\pm$ 0.9\hphantom{\tnote{$\downarrow$}} &
                94.7 $\pm$ 0.1\hphantom{\tnote{$\downarrow$}} &
                70.5 &
                \texttt{ 0} \\
            IGA~\cite{koyama2020out} &
                80.9 $\pm$ 0.4\tnote{$\downarrow$} &
                63.6 $\pm$ 0.2\tnote{$\uparrow$} &
                41.3 $\pm$ 0.8\tnote{$\downarrow$} &
                95.1 $\pm$ 0.1\tnote{$\uparrow$} &
                70.2 &
                \texttt{ 0} \\
            CORAL~\cite{sun2016deep} &
                81.6 $\pm$ 0.6\tnote{$\uparrow$} &
                63.8 $\pm$ 0.3\tnote{$\uparrow$} &
                38.3 $\pm$ 0.7\tnote{$\downarrow$} &
                94.2 $\pm$ 0.3\tnote{$\downarrow$} &
                69.5 &
                \texttt{ 0} \\ 
            IRM~\cite{arjovsky2019invariant} &
                81.1 $\pm$ 0.3\tnote{$\downarrow$} &
                63.0 $\pm$ 0.2\tnote{$\downarrow$} &
                42.0 $\pm$ 1.8\hphantom{\tnote{$\downarrow$}} &
                95.0 $\pm$ 0.4\tnote{$\uparrow$} &
                70.3 &
                \texttt{-1} \\
            VREx~\cite{krueger2020outofdistribution} &
                81.8 $\pm$ 0.1\tnote{$\uparrow$} &
                63.5 $\pm$ 0.1\hphantom{\tnote{$\downarrow$}} &
                40.7 $\pm$ 0.7\tnote{$\downarrow$} &
                94.1 $\pm$ 0.3\tnote{$\downarrow$} &
                70.0 &
                \texttt{-1} \\
            GroupDRO~\cite{sagawa2019distributionally} &
                80.4 $\pm$ 0.3\tnote{$\downarrow$} &
                63.2 $\pm$ 0.2\hphantom{\tnote{$\downarrow$}} &
                36.8 $\pm$ 1.1\tnote{$\downarrow$} &
                95.2 $\pm$ 0.2\tnote{$\uparrow$} &
                68.9 &
                \texttt{-1} \\
            ERDG~\cite{zhao2020domain} &
                80.5 $\pm$ 0.5\tnote{$\downarrow$} &
                63.0 $\pm$ 0.4\tnote{$\downarrow$} &
                41.3 $\pm$ 1.2\tnote{$\downarrow$} &
                95.5 $\pm$ 0.2\tnote{$\uparrow$}   &
                70.1 &
                \texttt{-2} \\
            DANN~\cite{ganin2016domain} &
                81.1 $\pm$ 0.4\tnote{$\downarrow$} &
                62.9 $\pm$ 0.6\tnote{$\downarrow$} &
                39.5 $\pm$ 0.2\tnote{$\downarrow$} &
                94.9 $\pm$ 0.0\tnote{$\uparrow$} &
                69.6 &
                \texttt{-2} \\
            MTL~\cite{blanchard2017domain} &
                81.2 $\pm$ 0.4\tnote{$\downarrow$} &
                62.9 $\pm$ 0.2\tnote{$\downarrow$} &
                38.9 $\pm$ 0.6\tnote{$\downarrow$} &
                95.0 $\pm$ 0.1\tnote{$\uparrow$} &
                69.5 &
                \texttt{-2} \\
            Mixup~\cite{yan2020improve} &
                79.8 $\pm$ 0.6\tnote{$\downarrow$} &
                63.3 $\pm$ 0.5\hphantom{\tnote{$\downarrow$}} &
                39.8 $\pm$ 0.3\tnote{$\downarrow$} &
                94.6 $\pm$ 0.3\hphantom{\tnote{$\downarrow$}} &
                69.4 &
                \texttt{-2} \\
            ANDMask~\cite{parascandolo2021learning} &
                79.5 $\pm$ 0.0\tnote{$\downarrow$} &
                62.0 $\pm$ 0.3\tnote{$\downarrow$} &
                39.8 $\pm$ 1.4\tnote{$\downarrow$}&
                95.3 $\pm$ 0.1\tnote{$\uparrow$} &
                69.2 &
                \texttt{-2} \\
            ARM~\cite{zhang2020adaptive} &
                81.0 $\pm$ 0.4\tnote{$\downarrow$} &
                63.2 $\pm$ 0.2\hphantom{\tnote{$\downarrow$}} &
                39.4 $\pm$ 0.7\tnote{$\downarrow$} &
                93.5 $\pm$ 0.6\tnote{$\downarrow$} &
                69.3 &
                \texttt{-3} \\
            MLDG~\cite{li2018learning} &
                73.0 $\pm$ 0.4\tnote{$\downarrow$} &
                52.4 $\pm$ 0.2\tnote{$\downarrow$} &
                27.4 $\pm$ 2.0\tnote{$\downarrow$} &
                91.2 $\pm$ 0.4\tnote{$\downarrow$} &
                61.0 &
                \texttt{-4} \\
            \midrule
            \textbf{Average} & 80.7 & 62.5 & 39.8 & 94.6 & 69.4 & -- \\
            \bottomrule
        \end{tabularx}
    \end{threeparttable}
    \caption{Performance of ERM and OoD generalization algorithms on datasets \emph{dominated by diversity shift}. Every symbol $\downarrow$ denotes a score of \texttt{-1}, and every symbol $\uparrow$ denotes a score of \texttt{+1}; otherwise the score is \texttt{0}. Adding up the scores across all datasets produces the ranking score for each algorithm.}
    \label{tab:benchmark-div}
    \vspace{-1mm}
\end{table*}


\paragraph{Datasets.}
In our experiment, datasets are chosen to cover as much variety from different OoD research areas as possible. As mentioned earlier, the datasets demonstrated two-dimensional properties shown by their estimated diversity and correlation shift.
The following datasets are dominated by diversity shift: \textbf{PACS}~\cite{li2017deeper}, \textbf{OfficeHome}~\cite{venkateswara2017deep},
\textbf{Terra Incognita}~\cite{beery2018recognition},
and \textbf{Camelyon17-WILDS}~\cite{koh2020wilds}.
On the other hand, our benchmark also include three datasets dominated by correlation shift:
\textbf{Colored MNIST}~\cite{arjovsky2019invariant},
\textbf{NICO}~\cite{he2020towards},
and a modified version of \textbf{CelebA}~\cite{liu2015deep}.
See Appendix~\ref{app:datasets} for more detailed descriptions of the above datasets.

For PACS, OfficeHome, and Terra Incognita, we train multiple models in every run with each treating one of the domains as the test environment and the rest of the domains as the training environments since it is common practice for DG datasets.
The final accuracy is the mean accuracy over all such splits.
For other datasets, the training and test environments are fixed.
A reason is that the leave-one-domain-out evaluation scheme would destroy the designated training/test splits of these datasets.
For more details about dataset statistics and environment splits, see Appendix~\ref{app:datasets}.

\vspace{-1mm}
\paragraph{Algorithms.}
We have selected Empirical Risk Minimization (\textbf{ERM})~\cite{vapnik1998statistical} and several representative algorithms from different OoD research areas for our benchmark:
Group Distributionally Robust Optimization (\textbf{GroupDRO})~\cite{sagawa2019distributionally},
Inter-domain Mixup (\textbf{Mixup})~\cite{xu2020adversarial, yan2020improve},
Meta-Learning for Domain Generalization (\textbf{MLDG})~\cite{li2018learning},
Domain-Adversarial Neural Networks (\textbf{DANN})~\cite{ganin2016domain},
Deep Correlation Alignment (\textbf{CORAL})~\cite{sun2016deep},
Maximum Mean Discrepancy (\textbf{MMD})~\cite{li2018domain},
Invariant Risk Minimization (\textbf{IRM})~\cite{arjovsky2019invariant},
Variance Risk Extrapolation (\textbf{VREx})~\cite{krueger2020outofdistribution},
Adaptive Risk Minimization (\textbf{ARM})~\cite{zhang2020adaptive},
Marginal Transfer Learning (\textbf{MTL})~\cite{blanchard2017domain},
Style-Agnostic Networks (\textbf{SagNet})~\cite{nam2019reducing},
Representation Self Challenging (\textbf{RSC})~\cite{huang2020self},
Learning Explanations that are Hard to Vary (\textbf{ANDMask})~\cite{parascandolo2021learning},
Out-of-Distribution Generalization with Maximal Invariant Predictor (\textbf{IGA})~\cite{koyama2020out},
and
Entropy Regularization for Domain Generalization (\textbf{ERDG})~\cite{zhao2020domain}.


\begin{table*}[t]
    \centering
	\begin{threeparttable}
        \begin{tabularx}{\linewidth}{l*{3}{X}rrr}
            \toprule
            \textbf{Algorithm} & \textbf{Colored MNIST} & \textbf{CelebA} & \textbf{NICO} & \textbf{Average} & \textbf{Prev score} & \textbf{Ranking score}\\
            \midrule
            VREx~\cite{krueger2020outofdistribution} &
                56.3 $\pm$ 1.9\tnote{$\uparrow$} &
                87.3 $\pm$ 0.2 &
                71.5 $\pm$ 2.3 &
                71.7 &
                \texttt{-1} &
                \texttt{+1} \\
            GroupDRO~\cite{sagawa2019distributionally} &
                32.5 $\pm$ 0.2\tnote{$\uparrow$} &
                87.5 $\pm$ 1.1 &
                71.0 $\pm$ 0.4 &
                63.7 &
                \texttt{-1} &
                \texttt{+1} \\
            \textbf{ERM}\vphantom{\tnote{$\downarrow$}}~\cite{vapnik1998statistical} & 
                29.9 $\pm$ 0.9 &
                87.2 $\pm$ 0.6 &
                72.1 $\pm$ 1.6 &
                63.1 &
                \texttt{ 0} &
                \texttt{ 0} \\
            IRM~\cite{arjovsky2019invariant} & 
                60.2 $\pm$ 2.4\tnote{$\uparrow$} &
                85.4 $\pm$ 1.2\tnote{$\downarrow$} &
                73.3 $\pm$ 2.1 &
                73.0 &
                \texttt{-1} &
                \texttt{ 0} \\
            MTL\vphantom{\tnote{$\downarrow$}}~\cite{blanchard2017domain} &
                29.3 $\pm$ 0.1 &
                87.0 $\pm$ 0.7 &
                70.6 $\pm$ 0.8 &
                62.3 &
                \texttt{-2} &
                \texttt{ 0} \\
            ERDG~\cite{zhao2020domain} &
                31.6 $\pm$ 1.3\tnote{$\uparrow$} &
                84.5 $\pm$ 0.2\tnote{$\downarrow$} &
                72.7 $\pm$ 1.9 &
                62.9 &
                \texttt{-2} &
                \texttt{ 0} \\
            ARM~\cite{zhang2020adaptive} & 
                34.6 $\pm$ 1.8\tnote{$\uparrow$} &
                86.6 $\pm$ 0.7 &
                67.3 $\pm$ 0.2\tnote{$\downarrow$} &
                62.8 &
                \texttt{-3} &
                \texttt{ 0} \\
            MMD~\cite{li2018domain} &
                50.7 $\pm$ 0.1\tnote{$\uparrow$} &
                86.0 $\pm$ 0.5\tnote{$\downarrow$} &
                68.9 $\pm$ 1.2\tnote{$\downarrow$} &
                68.5 &
                \texttt{+2} &
                \texttt{-1} \\
            RSC~\cite{huang2020self} &
                28.6 $\pm$ 1.5\tnote{$\downarrow$} &
                85.9 $\pm$ 0.2\tnote{$\downarrow$} &
                74.3 $\pm$ 1.9\tnote{$\uparrow$} &
                61.4 &
                \texttt{+2} &
                \texttt{-1} \\
            IGA~\cite{koyama2020out} &
                29.7 $\pm$ 0.5 &
                86.2 $\pm$ 0.7\tnote{$\downarrow$} &
                71.0 $\pm$ 0.1 &
                62.3 &
                \texttt{ 0} &
                \texttt{-1} \\
            CORAL~\cite{sun2016deep} &
                30.0 $\pm$ 0.5 &
                86.3 $\pm$ 0.5\tnote{$\downarrow$} &
                70.8 $\pm$ 1.0 &
                61.5 &
                \texttt{-1} &
                \texttt{-1} \\
            Mixup~\cite{yan2020improve} &
                27.6 $\pm$ 1.8\tnote{$\downarrow$} &
                87.5 $\pm$ 0.5 &
                72.5 $\pm$ 1.5 &
                60.6 &
                \texttt{-2} &
                \texttt{-1} \\
            MLDG~\cite{li2018learning}  &
                32.7 $\pm$ 1.1\tnote{$\uparrow$} &
                85.4 $\pm$ 1.3\tnote{$\downarrow$} &
                66.6 $\pm$ 2.4\tnote{$\downarrow$} &
                56.6 &
                \texttt{-4} &
                \texttt{-1} \\
            SagNet~\cite{nam2019reducing} &
                30.5 $\pm$ 0.7 &
                85.8 $\pm$ 1.4\tnote{$\downarrow$} &
                69.8 $\pm$ 0.7\tnote{$\downarrow$} &
                62.0 &
                \texttt{+1} &
                \texttt{-2} \\
            ANDMask~\cite{parascandolo2021learning} &
                27.2 $\pm$ 1.4\tnote{$\downarrow$} &
                86.2 $\pm$ 0.2\tnote{$\downarrow$} &
                71.2 $\pm$ 0.8 &
                61.5 &
                \texttt{-2} &
                \texttt{-2} \\
            DANN~\cite{ganin2016domain} & 
                24.5 $\pm$ 0.8\tnote{$\downarrow$} &
                86.0 $\pm$ 0.4\tnote{$\downarrow$} &
                69.4 $\pm$ 1.7\tnote{$\downarrow$} &
                59.7 &
                \texttt{-2} &
                \texttt{-3} \\
            \midrule
            \textbf{Average} & 34.5 & 86.4 & 70.8 & 63.7 & -- & -- \\
            \bottomrule
        \end{tabularx}
    \end{threeparttable}
    \caption{Performance of ERM and OoD generalization algorithms on datasets \emph{dominated by correlation shift}. Every symbol $\downarrow$ denotes a score of \texttt{-1}, and every symbol $\uparrow$ denotes a score of \texttt{+1}; otherwise the score is \texttt{0}. Adding up the scores across all datasets produces the ranking score for each algorithm. Prev scores are the scores of corresponding algorithms in \cref{tab:benchmark-div}.}
    \label{tab:benchmark-cor}
    \vspace{-1mm}
\end{table*}

\vspace{-1mm}
\paragraph{Model selection methods.}
As there is still no consensus on what model selection methods should be used in OoD generalization research \cite{gulrajani2021in}, appropriate selection methods are chosen for each dataset in our study.
To be consistent with existing lines of work~\cite{li2017deeper, carlucci2019domain, nam2019reducing, huang2020self, krueger2020outofdistribution}, models trained on PACS, OfficeHome, and Terra Incognita are selected by \emph{training-domain validation}.
As for Camelyon17-WILDS and NICO, \emph{OoD validation} is adopted in respect of \cite{koh2020wilds} and \cite{bai2020decaug}. The two remaining datasets, Colored MNIST and CelebA, use \emph{test-domain validation} which has been seen in  \cite{arjovsky2019invariant, krueger2020outofdistribution, ahuja2020invariant, pezeshki2020gradient}. Another reason for using test-domain validation is that it may be improper to apply training-domain validation to datasets dominated by correlation shift since under the influence of spurious correlations, achieving excessively high accuracy in the training environments often leads to low accuracy in novel test environments.
More detailed explanations of these model selection methods are provided in Appendix~\ref{app:model-selection}.

\vspace{-1mm}
\paragraph{Implementation details.}
Unlike DomainBed, we use a simpler model, ResNet-18 \cite{he2016deep}, for all algorithms and datasets excluding Colored MNIST, as it is the common practice in previous works \cite{huang2020self, nam2019reducing, zhao2020domain, dou2019domain, carlucci2019domain}.
Moreover, we believe smaller models could enlarge the gaps in OoD generalization performance among the algorithms, as larger models are generally more robust to OoD data \cite{hendrycks2020many} and thus the performance is easier to saturate on small datasets.
The ResNet-18 is pretrained on ImageNet and then finetuned on each dataset with only one exception---NICO, which contains photos of animals and vehicles largely overlapped with ImageNet classes. For simplicity, we continue to use a two-layer perceptron following \cite{arjovsky2019invariant, krueger2020outofdistribution, pezeshki2020gradient} for Colored MNIST.
Our experiments further differs from DomainBed in several minor aspects. First, we do not freeze any batch normalization layer in ResNet-18, nor do we use any dropout, to be consistent with most of prior works in DG.
Second, we use a larger portion (90\%) of data from training environments for training and the rest for validation. Third, we use a slightly different data augmentation scheme following \cite{carlucci2019domain}.

Finally, we adopt the following hyperparameter search protocol, the same as in DomainBed:
a 20-times random search is conducted for every pair of dataset and algorithm, and then the search process is repeated for another two random series of hyperparameter combinations, weight initialization, and dataset splits.
Altogether, the three series yield the three best accuracies over which a mean and standard error bar is computed for every dataset-algorithm pair.
See Appendix~\ref{app:hparams-search} for the hyperparameter search space for every individual algorithm.


\vspace{-1mm}
\paragraph{Results.}
The benchmark results are shown in \Cref{tab:benchmark-div} and \Cref{tab:benchmark-cor}.
In addition to mean accuracy and standard error bar, we report a ranking score for each algorithm with respect to ERM. For every dataset-algorithm pair, depending on whether the attained accuracy is lower than, within, or higher than the standard error bar of ERM accuracy on the same dataset, we assign score \texttt{-1}, \texttt{0}, \texttt{+1} to the pair.
Adding up the scores across all datasets produces the ranking score for each algorithm. We underline that the ranking score does not indicate whether an algorithm is definitely better or worse than the other algorithms.
It only reflects a relative degree of robustness against diversity and correlation shift.

\begin{figure*}[t]
    \centering
    \begin{subfigure}[b]{0.30\textwidth}
        \centering
        \includegraphics[width=\linewidth]{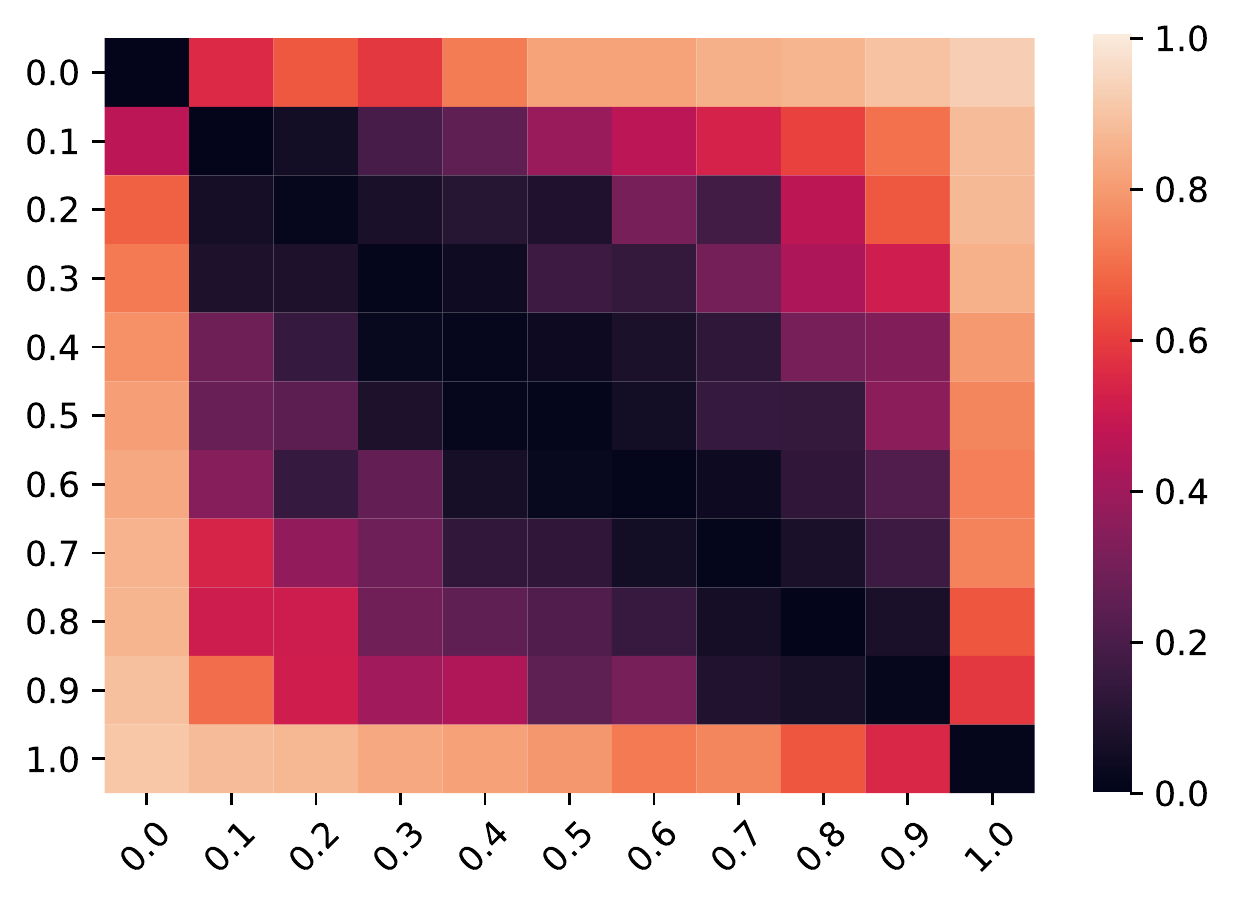}
        \caption{Estimation of correlation shift under varying $\rho_\textnormal{tr}$ and $\rho_\textnormal{te}$. Digits are in red and green only.\newline}
        \label{fig:cor-ablation-1}
    \end{subfigure}
    \hfill
    \begin{subfigure}[b]{0.30\textwidth}
        \centering
        \includegraphics[width=\linewidth]{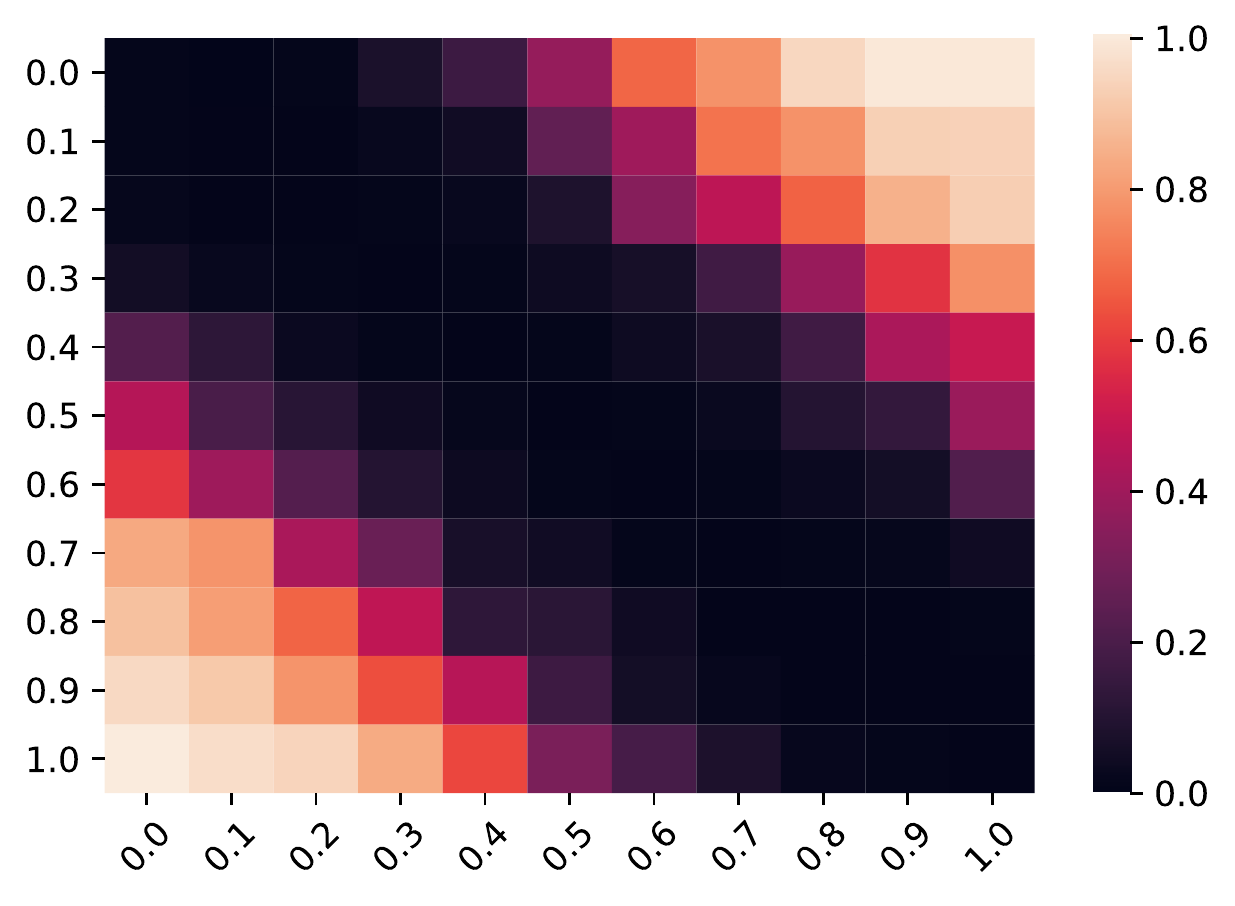}
        \caption{Estimation of diversity shift under varying $\mu_\textnormal{tr}$ and $\mu_\textnormal{te}$ while fixing $\rho_\textnormal{tr} = 0.1$, $\rho_\textnormal{te} = 0.9$ and $\sigma_\textnormal{tr} = \sigma_\textnormal{te} = 0.1$.}
        \label{fig:div-ablation-1}
    \end{subfigure}
    \hfill
    \begin{subfigure}[b]{0.30\textwidth}
        \centering
        \includegraphics[width=\linewidth]{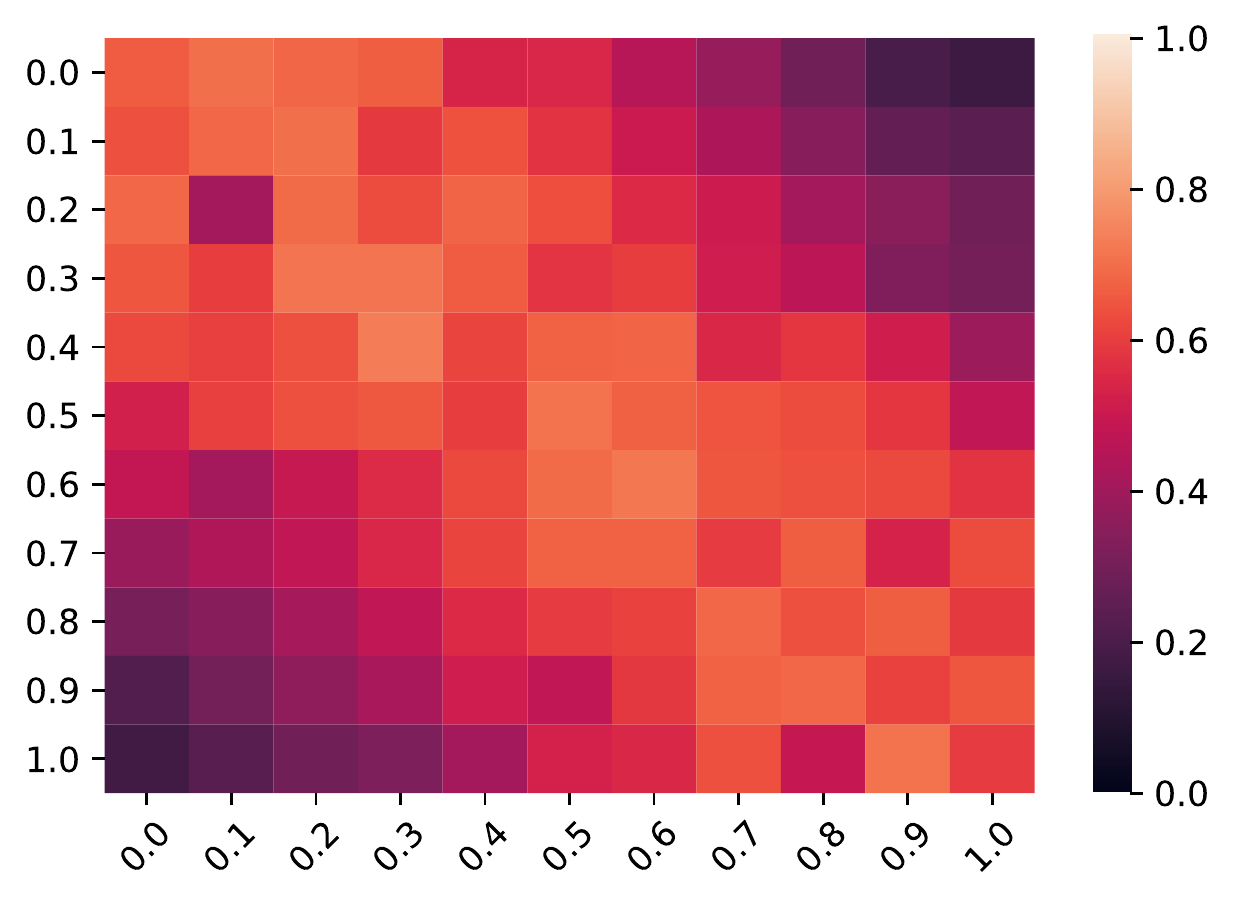}
        \caption{Estimation of correlation shift under varying $\mu_\textnormal{tr}$ and $\mu_\textnormal{te}$ while fixing $\rho_\textnormal{tr} = 0.1$, $\rho_\textnormal{te} = 0.9$ and $\sigma_\textnormal{tr} = \sigma_\textnormal{te} = 0.25$.}
        \label{fig:cor-ablation-2}
    \end{subfigure}
    \caption{Estimation of diversity and correlation shift under varying color distribution in Colored MNIST. Another color, blue, uncorrelated with the labeled classes, is added onto the digits to create diversity shift. The intensity of blue is sampled from a truncated Gaussian distribution for every image. Assuming only one training and one test environment, $\rho_\textnormal{tr}$ and $\rho_\textnormal{te}$ stand for the correlation between red/green and the digits; $\mu_\textnormal{tr}$ and $\mu_\textnormal{te}$ stand for the mean intensities of blue; $\sigma_\textnormal{tr}$ and $\sigma_\textnormal{te}$ stand for the standard deviations.}
    \label{fig:vary_div_cor}
    \vspace{-1mm}
\end{figure*}

From \Cref{tab:benchmark-div} and \Cref{tab:benchmark-cor}, we observe that none of the OoD generalization algorithms achieves consistently better performance than ERM on both OoD directions. For example, on the datasets dominated by diversity shift, the ranking scores of RSC, MMD, and SagNet are higher than ERM, whereas on the datasets dominated by correlation shift, their scores are lower.
Conversely, the algorithms (VREx and GroupDRO) that outperform ERM in \Cref{tab:benchmark-cor} are worse than ERM on datasets of the other kind. This supports our view that \emph{OoD generalization algorithms should be evaluated on datasets embodying both diversity and correlation shift}. Such a comprehensive evaluation is of great importance because real-world data could be tainted by both kinds of distribution shift, \eg, the ImageNet variants in \Cref{fig:estimation}.

In the toy case of Colored MNIST, several algorithms are superior to ERM, however, in the more realistic and complicated cases of CelebA and NICO, none of the algorithms surpasses ERM by a large margin.
Hence, we argue that \emph{contemporary OoD generalization algorithms are by large still vulnerable to spurious correlations}. In particular, IRM that achieves the best accuracy on Colored MNIST among all algorithms, fails to surpass ERM on the other two datasets.
It is in line with the theoretical results discovered by \cite{rosenfeld2021the}: IRM does not improve over ERM unless the test data are sufficiently similar to the training distribution.
Besides, we have also done some experiments on ImageNet-V2, and the result again supports our argument (see \cref{app:imagenetv2-exp}).

Due to inevitable noises and other changing factors in the chosen datasets and training process, whether there is any compelling pattern in the results across the datasets dominated by the same kind of distribution shift is unclear. So, it is important to point out that the magnitude of diversity and correlation shift does not indicate the absolute level of difficulty for generalization. Instead, it represents a likelihood that certain algorithms will perform better than some other algorithms under the same kind of distribution shift.

\subsection{Further Study}
In this section, we conduct further experiments to check the reliability of our estimation method for diversity and correlation shift and compare our method against other existing metrics for measuring the non-i.i.d\onedot property of datasets, demonstrating the robustness of our estimation method and the significance of diversity and correlation shift.

\vspace{-1mm}
\paragraph{Sanity check and numerical stability.}
To validate the robustness of our estimation method, we check whether it can produce stable results that faithfully reflect the expected trend as we manipulate the color distribution of Colored MNIST. For simplicity, only one training environment is assumed.
To start with, we manipulate the correlation coefficients $\rho_\textnormal{tr}$ and $\rho_\textnormal{te}$ between digits and colors in constructing the dataset. From \Cref{fig:cor-ablation-1}, we can observe that when $\rho_\textnormal{tr}$ and $\rho_\textnormal{te}$ have similar values, the estimated correlation shift is negligible. It aligns well with our definition of correlation shift that measures the distribution difference of features present in both environments.
As for examining on the estimation of diversity shift, another color, blue, is introduced in the dataset.
The intensity (between $0$ and $1$) of blue added onto each digit is sampled from truncated Gaussian distributions with means $\mu_\textnormal{tr}$, $\mu_\textnormal{te}$ and standard deviations $\sigma_\textnormal{tr}$, $\sigma_\textnormal{te}$ for training and test environment respectively. Meanwhile, the intensity of red and green is subtracted by the same amount.
From \Cref{fig:div-ablation-1}, we observe that as the difference in color varies between red/green and blue, the estimate of diversity shift varies accordingly (at the corners).
Lastly, we investigate the behavior in the estimation of correlation shift while keeping the correlation coefficients fixed and manipulating $\mu_\textnormal{tr}$ and $\mu_\textnormal{te}$ that controls diversity shift. \Cref{fig:cor-ablation-2} shows a trade-off between diversity and correlation shift, as implied by their definitions. Experiments in every grid cell are conducted only once, so the heatmaps also reflect the variance in our estimation, which can be compensated by averaging over multiple runs.

\vspace{-1mm}
\paragraph{Comparison with other measures of distribution shift.}
We also compare OoD-Bench with other measures of distribution shift. The results on the variants of Colored MNIST are shown in \Cref{tab:failure-cases}. We empirically show that general metrics for measuring the discrepancy between distributions, such as EMD~\cite{rubner1998metric} and MMD~\cite{gretton2006kernel}, are not very informative.
Specifically, EMD and MMD are insensitive to the correlation shift in the datasets, while EMD is also insensitive to the diversity shift.
Although NI \cite{he2020towards} can produce comparative results on correlation shift, it is still unidimensional like EMD and MMD, not discerning the two kinds of distribution shift present in the datasets.
In comparison, our method provides more stable and interpretable results.
As $\rho_\textnormal{tr}$ an $\rho_\textnormal{te}$ gradually become close, the estimated correlation shift reduces to zero.
On the other hand, the estimated diversity shift remains constant zero until the last scenario where our method again produces the expected answer.

\section{Related Work}
\begin{table*}[t]
    \small
    \centering
    \begin{tabular}{l c c c c c c}
        \toprule
        \textbf{$\rho_\textnormal{te}$} &\makecell{\textbf{Dominant}\\\textbf{shift}} & \textbf{EMD}\hphantom{\Checkmark} & \textbf{MMD}\hphantom{\Checkmark} & \textbf{NI}\hphantom{\Checkmark} & \makecell{\textbf{Div\onedot shift}\\{(ours)}} & \makecell{\textbf{Cor\onedot shift}\\{(ours)}} \\
        \midrule
        0.9      & Cor\onedot shift & 0.08 $\pm$ 0.01 \XSolidBrush & 0.01 $\pm$ 0.00 \XSolidBrush & 1.40 $\pm$ 0.06 \Checkmark & 0.00 $\pm$ 0.00 & 0.67 $\pm$ 0.04 \\
        0.7      & Cor\onedot shift & 0.07 $\pm$ 0.00 \XSolidBrush & 0.01 $\pm$ 0.00 \XSolidBrush & 1.05 $\pm$ 0.03 \Checkmark & 0.00 $\pm$ 0.00 & 0.48 $\pm$ 0.06 \\
        0.5      & Cor\onedot shift & 0.07 $\pm$ 0.00 \XSolidBrush & 0.00 $\pm$ 0.00 \XSolidBrush & 0.72 $\pm$ 0.04 \Checkmark & 0.00 $\pm$ 0.00 & 0.34 $\pm$ 0.06 \\
        0.3      & Cor\onedot shift & 0.06 $\pm$ 0.00 \XSolidBrush & 0.00 $\pm$ 0.00 \XSolidBrush & 0.57 $\pm$ 0.04 \Checkmark & 0.00 $\pm$ 0.00 & 0.18 $\pm$ 0.05 \\
        0.1      & None       & 0.06 $\pm$ 0.00 \XSolidBrush & 0.00 $\pm$ 0.00 \Checkmark & 0.39 $\pm$ 0.02 \XSolidBrush & 0.00 $\pm$ 0.00 & 0.00 $\pm$ 0.00 \\
        0.1$^\dagger$ & Div\onedot shift & 0.29 $\pm$ 0.01 \XSolidBrush & 1.00 $\pm$ 0.00 \Checkmark & 10.76 $\pm$ 0.43 \XSolidBrush\hphantom{0} & 0.93 $\pm$ 0.01 & 0.00 $\pm$ 0.00 \\
        \bottomrule
    \end{tabular}
    \caption{Existing metrics on measuring the distribution shift in Colored MNIST with only one training environment where $\rho_\textnormal{tr} = \textnormal{0.1}$. All environments contain only red and green digits except the last. $^\dagger$Blue is added with $\mu_\textnormal{tr} = \textnormal{0}$, $\mu_\textnormal{te} = \textnormal{1}$ and $\sigma_\textnormal{tr} = \sigma_\textnormal{te} = \textnormal{0.1}$. Results are averaged over 5 runs.}
    \label{tab:failure-cases}
    \vspace{-1mm}
\end{table*}

\paragraph{Quantification on distribution shifts.}
Non-i.i.d\onedot Index (NI) \cite{he2020towards} quantifies the degree of distribution shift between training and test set with a single formula.
There are also a great number of general distance measures for distributions: Kullback-Leibler (KL) divergence, EMD~\cite{rubner1998metric}, MMD~\cite{gretton2006kernel}, and $\Acal$-distance~\cite{ben2007analysis}, etc.
However, they all suffer from the same limitation as NI, not being able to discern different kinds of distribution shifts.
To the best of our knowledge, we are among the first to formally identify the two-dimensional distribution shift and provide quantitative results on various OoD datasets.
Notably, a concurrent work~\cite{wiles2022a} studies three kinds of distribution shift, namely \emph{spurious correlation}, \emph{low-data drift}, and \emph{unseen data shift}, which are very similar to correlation and diversity shift.
Their findings are mostly in line with ours, but they do not provide any quantification formula or estimation method for the shifts.

\vspace{-1.5mm}
\paragraph{OoD generalization.}
Without access to test distribution examples, OoD generalization always requires additional assumptions or domain information.
In the setting of DG~\cite{blanchard2011generalizing,torralba2011unbiased,muandet2013domain}, it is often assumed that multiple training datasets sampled from similar but distinct domains are available.
Hence, most DG algorithms aim at learning a domain-invariant data representation across training domains.
These algorithms take various approaches include domain adversarial learning~\cite{ganin2016domain, li2018domain, akuzawa2019adversarial, albuquerque2019adversarial, albuquerque2019generalizing, zhao2020domain, yan2020improve, xu2020adversarial},
meta-learning~\cite{li2018learning, balaji2018metareg, dou2019domain, li2019feature, zhang2020adaptive},
image-level and feature-level domain mixup~\cite{xu2020adversarial, mancini2020towards},
adversarial data augmentation~\cite{shankar2018generalizing},
domain translation/randomization~\cite{nguyen2021domain, robey2021model, zhou2020learning},
feature alignment~\cite{sun2016deep, peng2019moment},
gradient alignment~\cite{koyama2020out, rame2021fishr, shi2021gradient},
gradient orthogonalization~\cite{bai2020decaug},
invariant risk minimization~\cite{arjovsky2019invariant, ahuja2020invariant, krueger2020outofdistribution},
self-supervised learning~\cite{wang2020learning, zhou2020deep},
prototypical learning~\cite{dubey2021adaptive},
and kernel methods~\cite{blanchard2017domain, muandet2013domain, li2018domain2, ghifary2016scatter}.
There are also DG algorithms that do not assume multiple training domains.
Many of them instead assume that variations in the style/texture of images is the main cause of distribution shift.
These algorithms mostly utilize AdaIN~\cite{huang2017arbitrary} or similar operations to perform style perturbations so that the learned classifier would be invariant to various styles across domains~\cite{nam2019reducing, li2021progressive, somavarapu2020frustratingly, wang2021learning, zhou2020domain, jeon2021feature}.
Other approaches include \cite{carlucci2019domain} which designs a self-supervision objective enforcing models to focus on global image structures such as shapes of objects, and \cite{wang2019learning} which introduces an explicit adversarial learning objective so that the learned model would be invariant to local patterns.
More general single-source DG algorithms (that do not assume the style/texture bias) and other OoD generalization algorithms include distributionally robust optimization~\cite{sagawa2019distributionally}, self-challenging~\cite{huang2020self}, spectral decoupling~\cite{pezeshki2020gradient}, feature augmentation~\cite{li2021simple}, adversarial data augmentation~\cite{volpi2018generalizing, qiao2020learning}, gradient alignment~\cite{parascandolo2021learning}, sample reweighting~\cite{he2020towards, zhang2021deep}, test-time training~\cite{sun2020test}, removing bias with bias~\cite{bahng2019rebias}, contrastive learning~\cite{kim2021selfreg}, causal discovery~\cite{mouli2021asymmetry}, and variational bayes that leverages causal structures of data~\cite{liu2021learning, sun2021recovering}.
For a more comprehensive summary of existing OoD generalization and DG algorithms, we refer readers to these survey papers~\cite{zhou2021domain, shen2021towards, wang2021generalizing}.


\vspace{-1.5mm}
\paragraph{DomainBed.}
The living benchmark is created by \cite{gulrajani2021in} to facilitate disciplined and reproducible DG research.
After conducting a large-scale hyperparameter search, the performances of fourteen algorithms on seven datasets are reported.
The authors then arrive at the conclusion that ERM beats most of DG algorithms under the same fair setting.
Our work differs from DomainBed mainly in three aspects.
First, we not only provide a benchmark for algorithms but also for datasets, helping us gain a deeper understanding of the distribution shift in the data.
Second, we compare different algorithms in a more informative manner in light of diversity and correlation shift, recovering the fact that some algorithms are indeed better than ERM in appropriate scenarios.
Third, we experiment with several new algorithms and new datasets, especially those dominated by correlation shift.

\section{Conclusion}

In this paper, we have identified diversity shift and correlation shift as two of the main forms of distribution shift in OoD datasets. The two-dimensional characterization positions disconnected datasets into a unified picture and have shed light on the nature of unknown distribution shift in some real-world data. In addition, we have demonstrated some of the strengths and weaknesses of existing OoD generalization algorithms.
The results suggest that future algorithms should be more comprehensively evaluated on two types of datasets, one dominated by diversity shift and the other dominated by correlation shift.
Lastly, we leave an open problem regarding whether there exists an algorithm that can perform well under both diversity and correlation shift. If not then our method can be used for choosing the appropriate algorithms.

{\small
\bibliographystyle{ieee_fullname}
\bibliography{ood}
}


\newpage
\appendix
\onecolumn

\section{Acknowledgements}
Nanyang Ye was supported by National Natural Science Foundation of China under Grant 62106139, in part by National Key R\&D Program of China 2018AAA0101200, in part by National Natural Science Foundation of China under Grant (No. 61829201, 61832013, 61960206002, 62061146002,  42050105, 62032020), in part by the Science and Technology Innovation Program of Shanghai (Grant 18XD1401800),  and in part by Shanghai Key Laboratory of Scalable Computing and Systems, and in part by BIREN Tech.

\section{Proofs}
\label{app:proofs}
\begin{proposition}
    \label{prop:zero-one-bound}
    For any non-negative probability functions $p$ and $q$ of two environments, diversity shift $D_\textnormal{div}(p, q)$ and correlation shift $D_\textnormal{cor}(p, q)$ are always bounded between 0 and 1, inclusively.
\end{proposition}
\begin{proof}
    Apparently, $D_\textnormal{div}(p, q)$ and $D_\textnormal{cor}(p, q)$ are always non-negative, so we are left to prove the upper bound. By the triangle inequality and that every probability function sums up to one over all possible outcomes, we have
    \begin{equation}
        \begin{split}
            D_\textnormal{div}(p, q) &= \frac{1}{2} \int_{\Scal} |\:\! p(\zB) - q(\zB) |\, d \zB \leq \frac{1}{2} \int_{\Scal} \big[ p(\zB) + q(\zB) \big] \, d \zB \leq 1.
        \end{split}
    \end{equation}
    Similarly, we also have
    \begin{align}
        D_\textnormal{cor}(p, q)
        &= \frac{1}{2} \int_{\Tcal} \sqrt{p(\zB)\,q(\zB)} \sum_{\yB \in \Ycal} {\left|\:\! p(\yB \given \zB) - q(\yB \given \zB) \right|} \, d \zB \nonumber\\
        &\leq \frac{1}{2} \int_{\Tcal} \sqrt{p(\zB)\,q(\zB)} \sum_{\yB \in \Ycal} { \big[ p(\yB \given \zB) + q(\yB \given \zB) \big] } \, d \zB \\
        &= \frac{1}{2} \int_{\Tcal} 2\sqrt{p(\zB)\,q(\zB)} \, d \zB
        \leq \frac{1}{2} \int_{\Tcal} \big[ p(\zB) + q(\zB) \big] \, d \zB \leq 1. \nonumber \qedhere
    \end{align}
\end{proof}

\begin{lemma}
    \label{lemma:alpha-x}
    Suppose there are equal amount of examples from two environments,
    then for any example $\xB \in \Xcal$ sampled from either of the environments, the probability $\alpha(\xB)$ of a prediction $t(\xB) \in \{ 0, 1 \}$ that predicts the sampling environment of $\xB$ being correct is
    \begin{equation*}
        \frac{(1 - t(\xB)) \cdot p(\xB) + t(\xB) \cdot q(\xB)}{p(\xB) + q(\xB)}.
    \end{equation*}
\end{lemma}
\begin{proof}
    First of all, we need to define several quantities. Let the probability of an example being sampled from the first environment be $P(E = 0)$ and the probability of an example being sampled from the second environment be $P(E = 1)$.
    Since there are equal amount of examples from both environment, we have $P(E=0) = P(E=1) = \frac{1}{2}$.
    The probability of an example (from one of the environments) taking on a particular value $\xB$ is $P(X=\xB \given E=0)$ and $P(X=\xB \given E=1)$.
    By definition, $P(X=\xB \given E=0) = p(\xB)$ and $P(X=\xB \given E=1) = q(\xB)$.
    The probability of an example taking on a particular value $\xB$ (regardless of the environment) is given by
    \begin{equation}
        \begin{split}
            P(X=\xB)
                &= P(X=\xB, E=0) + P(X=\xB, E=1) \\
                &= P(X=\xB \given E=0) \cdot P(E=0) + P(X=\xB \given E=1) \cdot P(E=1) = \frac{p(\xB) + q(\xB)}{2}.
        \end{split}
    \end{equation}
    The probability of a given example $\xB$ being sampled from the first environment is
    \begin{equation}
        \begin{split}
            P(E = 0 \given X=\xB)
                &= \frac{P(X=\xB \given E=0) \cdot P(E=0)}{P(X=\xB)}
                 = \frac{p(\xB)}{2P(X=\xB)}
                 = \frac{p(\xB)}{p(\xB) + q(\xB)}.
        \end{split}
    \end{equation}
    Similarly, the probability of $\xB$ being sampled from the second environment is
    \begin{equation}
        P(E = 1 \given X = \xB) = \frac{q(\xB)}{p(\xB) + q(\xB)}.
    \end{equation}
    Together, the overall probability of some prediction $t(\xB)$ being correct is
    \begin{equation}
        \alpha(\xB)
        = (1 - t(\xB)) \cdot P(E = 0 \given X = \xB) + t(\xB) \cdot P(E = 1 \given X = \xB)
        = \frac{(1 - t(\xB)) \cdot p(\xB) + t(\xB) \cdot q(\xB)}{p(\xB) + q(\xB)}.
    \end{equation}
\end{proof}

\feature*
\begin{proof}
    To formally state that the network attains optimal performance in classifying the environments, we note that for any example $\xB \in \Xcal$, the probability $\alpha(\xB)$ of a prediction $t(\xB) \in \{ 0, 1 \}$ being correct is
    \begin{equation}
        \frac{(1 - t(\xB)) \cdot p(\xB) + t(\xB) \cdot q(\xB)}{p(\xB) + q(\xB)}.
    \end{equation}
    This has been shown by Lemma~\ref{lemma:alpha-x}.
    Hence, the overall classification accuracy is given by
    \begin{equation}
        \Ebb_{X} [ \alpha(\xB) ] = \Ebb_{X} \left[ \frac{(1 - t(\xB)) \cdot p(\xB) + t(\xB) \cdot q(\xB)}{p(\xB) + q(\xB)} \right] = \frac{1}{2} \int_{\Xcal}{(1 - t(\xB)) \cdot p(\xB) + t(\xB) \cdot q(\xB)}.
    \end{equation}
    The best possible classification accuracy that can be attained by the network is determined by the environments alone, which is
    \begin{equation}
        \begin{split}
            \frac{1}{2} \int_{\Xcal} \max_{t} \left\{(1 - t(\xB)) \cdot p(\xB) + t(\xB) \cdot q(\xB)\right\}
            &= \frac{1}{2} \int_{\Xcal} \max \{p(\xB), q(\xB)\}.
        \end{split}
    \end{equation}
    The above maximum is attained by $t^\ast$ satisfying the following equations for every $\xB$ such that $p(\xB) \neq q(\xB)$:
    \begin{equation}
        \label{eq:t-star}
        t^\ast(\xB) =
        \begin{cases}
            0 & \text{if } p(\xB) > q(\xB) \\
            1 & \text{if } p(\xB) < q(\xB)
        \end{cases}.
    \end{equation}
    In other words, $t^\ast$ predicts the environment from which an example is more likely sampled.
    Suppose for some $\xB$ such that $p(\xB) > q(\xB)$ we have (i) $g(\xB) = \zB$ and (ii) $\Hat{p}(\yB, \zB) = \Hat{q}(\yB, \zB)$ for every $\yB \in \Ycal$, then there must exist some $\xB^\prime \neq \xB$ such that $p(\xB^\prime) < q(\xB^\prime)$ with $f(\xB^\prime) = f(\xB)$ and $g(\xB^\prime) = g(\xB)$.
    It follows that $t(\xB) = t(\xB^\prime)$ because $\xB$ and $\xB^\prime$ both map to the same $\yB$ and $\zB$, which makes no difference to $h$.
    Finally, it is clear to see that $t \neq t^\ast$ and therefore the network does not attain the optimal performance.
\end{proof}

\section{Practical Estimation of Diversity and Correlation Shift}
\label{app:estimation-method}

In this section, we provide complete pseudo codes of our estimation methods for diversity and correlation shift, supported by more theoretical justifications.

\subsection{Pseudo codes and supporting theoretical justifications}
\label{app:pseudo-codes}
\makeatletter
\newcommand{\multiline}[1]{%
  \begin{tabularx}{\dimexpr\linewidth-\ALG@thistlm}[t]{@{}X@{}}
    #1
  \end{tabularx}
}
\makeatother
\begin{algorithm}[H]\small
    \caption{\;Training procedure of feature extractor and environment classifier}
    \label{alg:train-extractor}
    \begin{algorithmic}[1]
        \Require Training environment $\Ecal_\textnormal{tr}$ and test environments $\Ecal_\textnormal{te}$; mini-batch size $N$; number of training steps $T$; loss function $\ell$.
        \Ensure Feature extractor $g: \Xcal \to \Fcal$; environment classifier $h: \Fcal \times \Ycal \to [0, 1]$.
        \State Initialize network parameters;
        \For{each training step $t \gets 1, \dots, T$}
            \State \multiline{sample a mini-batch of training examples $\{(\xB_i, \yB_i, e_i)\}_{i=1}^{N}$ from $\Ecal_\textnormal{tr}$ (indexed by $e_i = 0$) and $\Ecal_\textnormal{te}$ (indexed by $e_i = 1$) while ensuring equal sampling probability for the two environments and for every distinct value $\yB \in \Ycal$ in each environment;}
            \For{each example $(\xB_i, \yB_i, e_i)$ in the mini-batch}
                \State $\Hat{e}_i \gets h(g(\xB_i), \yB_i)$;
                \State compute loss $\ell (\Hat{e}_i, e_i)$ and back-propagate gradients;
            \EndFor
            \State update network parameters by the accumulated gradients, and then reset the gradients.
        \EndFor
    \end{algorithmic}
\end{algorithm}
\vspace{1.5\parskip}
The training procedure of our feature extractor $g$ and environment classifier $h$ is described in Algorithm~\ref{alg:train-extractor}.
Note that in line 3, we use sample reweighting to ensure the class balance in every environment so that the following assumption holds.

\begin{assumption}
    \label{assumption:pqy}
    For every $\yB \in \Ycal$, $p(\yB) = q(\yB) > 0$, \ie, there is no label shift.
\end{assumption}

Each environment defines a distribution over $\Xcal$.
The two environments share the same labeling rule $f: \Xcal \to \Ycal$.
The feature extractor $g: \Xcal \to \Fcal$ maps every input $\xB \in \Xcal$ to an $d$-dimensional feature vector $\zB \in \Fcal$.
Put it together, the labeling rule $f$, the feature extractor $g$ together with the two distributions over $\Xcal$ induces two probability functions $\Hat{p}$ and $\Hat{q}$ over $\Xcal \times \Ycal \times \Fcal$.

As mentioned in the paper, for a practical estimation of diversity and correlation shift, we first partition $\Fcal$ into
\begin{equation}
    \Scal^\prime \definedas \{ \zB \in \Fcal \mid \Hat{p}(\zB) \cdot \Hat{q}(\zB) = 0 \} \quad\textnormal{and}\quad
    \Tcal^\prime \definedas \{ \zB \in \Fcal \mid \Hat{p}(\zB) \cdot \Hat{q}(\zB) \neq 0 \},
\end{equation}
and then estimate the shifts by
\begin{align}
    \label{eq:div-estimator}
    D^\prime_\textnormal{div}(\Hat{p}, \Hat{q})
    &\definedas \frac{1}{2} \int_{\Scal^\prime} |\:\! \Hat{p}(\zB) - \Hat{q}(\zB) |\, d \zB, \\
    \label{eq:cor-estimator}
    D^\prime_\textnormal{cor}(\Hat{p}, \Hat{q})
    &\definedas \frac{1}{2} \int_{\Tcal^\prime} \sqrt{\Hat{p}(\zB)\,\Hat{q}(\zB)} \sum_{\yB \in \Ycal} {\left|\:\! \Hat{p}(\yB \given \zB) - \Hat{q}(\yB \given \zB) \right|} \, d \zB,
\end{align}
where we have simply replaced $(p, q)$ with their empirical estimates $(\Hat{p}, \Hat{q})$ and replaced $(\Scal, \Tcal)$ with $(\Scal^\prime, \Tcal^\prime)$ in Definition~\ref{def:div-cor}.
One might notice a slight difference between the definition of $(\Scal, \Tcal)$ and the definition of $(\Scal^\prime, \Tcal^\prime)$, which is that $(\Scal, \Tcal)$ are subsets of $\Zcal_2$, and therefore only contain non-causal features, whereas $(\Scal^\prime, \Tcal^\prime)$ are subsets of $\Fcal$, which could contain representations of both $\Zcal_1$ and $\Zcal_2$.
Fortunately, this is not an issue for two reasons.
First, the features in $\Scal^\prime$ have no shared support in the two environments, \ie $\Hat{p}(\zB) \cdot \Hat{q}(\zB) = 0$.
Recall that
\begin{equation}
    \label{eq:pqz1-app}
    p(\zB) \cdot q(\zB) \neq 0 \:\land\: \forall\,\yB \in \Ycal: p(\yB \given \zB) = q(\yB \given \zB)
\end{equation}
holds for every $\zB \in \Zcal_1$.
This suggests that we would have $\Hat{p}(\zB) \cdot \Hat{q}(\zB) \neq 0$ for every representation $\zB$ of $\Zcal_1$, and therefore the integral over $\Scal^\prime$ would exclude these features.
Second, \eqref{eq:pqz1-app} also suggests that the term $\left|\:\! \Hat{p}(\yB \given \zB) - \Hat{q}(\yB \given \zB) \right|$ would be relatively small for every $\yB \in \Ycal$ and representation $\zB$ of $\Zcal_1$, so the integral over $\Tcal^\prime$ will not be affected by representations of $\Zcal_1$.

Once $g$ is trained properly, inputs from training and test environments are all processed by $g$. Then the output features $\Fcal$ are gathered into $\Fcal_\textnormal{tr}$ and $\Fcal_\textnormal{te}$, which are used for estimating the shifts as in Algorithm~\ref{alg:est-shifts} below.

\begin{figure}[H]
\begin{minipage}{\columnwidth}
\begin{algorithm}[H]\small
    \caption{\;Estimation of diversity and correlation shift}
    \label{alg:est-shifts}
    \begin{algorithmic}[1]
        \Require Features $\Fcal_\textnormal{tr}$ and $\Fcal_\textnormal{te}$ from training and test environments; importance sampling size $M$; thresholds $\epsilon_\textnormal{div}$ and $\epsilon_\textnormal{cor}$.
        \Ensure Estimated diversity shift $D^\prime_\textnormal{div}$; estimated correlation shift $D^\prime_\textnormal{cor}$.
        \State \textcolor{gray}{\texttt{\# Prepare for the estimation}}
        \State $\Fcal \gets \Fcal_\textnormal{tr} \cup \Fcal_\textnormal{te}$;
        \State scale $\Fcal$ to zero mean and unit variance;
        \State $\Hat{w} \gets$ fit by KDE the distribution of $\Fcal$;
        \State $\Fcal_\textnormal{tr}^\prime, \Fcal_\textnormal{te}^\prime \gets$ split $\Fcal$ to recover the original partition;
        \State $\Hat{p}, \Hat{q} \gets$ fit by KDE the distributions of $\Fcal_\textnormal{tr}^\prime$ and $\Fcal_\textnormal{te}^\prime$;
        \State
        \State \textcolor{gray}{\texttt{\# Estimate diversity shift}}
        \State $D^\prime_\textnormal{div} \gets 0$;
        \For{$t \gets 1, \dots, M$}
            \State $\zB \gets$ sample from $\Hat{w}$;
            \If{$\Hat{p}(\zB) < \epsilon_\textnormal{div}$ \textbf{or} $\Hat{q}(\zB) < \epsilon_\textnormal{div}$}
                \State $D^\prime_\textnormal{div} \gets D^\prime_\textnormal{div} + |\;\! \Hat{p}(\zB) - \Hat{q}(\zB) | \mathbin{/} \Hat{w}(\zB)$;
            \EndIf
        \EndFor
        \State $D^\prime_\textnormal{div} \gets D^\prime_\textnormal{div} \mathbin{/} 2M$;
        \State
        \State \textcolor{gray}{\texttt{\# Estimate correlation shift}}
        \State $D^\prime_\textnormal{cor} \gets 0$;
        \For{each $\yB \in \Ycal$}
            \State $\Hat{p}_\yB, \Hat{q}_\yB \gets$ fit by KDE the distributions of the subsets of $\Fcal_\textnormal{tr}^\prime$ and $\Fcal_\textnormal{tr}^\prime$ that correspond to the inputs with label $\yB$;
            \For{$t \gets 1, \dots, M$}
                \State $\zB \gets$ sample from $\Hat{w}$;
                \If{$\Hat{p}(\zB) > \epsilon_\textnormal{cor}$ \textbf{and} $\Hat{q}(\zB) > \epsilon_\textnormal{cor}$}
                    \State $D^\prime_\textnormal{cor} \gets D^\prime_\textnormal{cor} + |\;\! \Hat{p}_\yB(\zB) \sqrt{\Hat{q}(\zB)/\Hat{p}(\zB)} - \Hat{q}_\yB(\zB) \sqrt{\Hat{p}(\zB)/\Hat{q}(\zB)}| \mathbin{/} \Hat{w}(\zB)$;
                \EndIf
            \EndFor
        \EndFor
        \State $D^\prime_\textnormal{cor} \gets D^\prime_\textnormal{cor} \mathbin{/} 2M|\Ycal|$.
    \end{algorithmic}
\end{algorithm}
\end{minipage}
\end{figure}

\subsection{Implementation details}
Same as the networks on which the algorithms in \Cref{sec:benchmark} are trained, we use MLP for Colored MNIST and ResNet-18 for other datasets as the feature extractors.
All ResNet-18 models are pretrained on ImageNet.
The feature extractors are optimized by Adam with a fixed learning rate $0.0003$ for $T=2000$ iterations. The batch size $N$ we used is $32$ for each environment, and we set the feature dimension $m=8$.
For every random data split, we keep 90\% data for training and use the rest 10\% data for validation. We choose the models maximizing the accuracy (in predicting the environments) on validation sets.
For datasets with multiple training environments and test environments, the network is trained to discriminate all the environments.
The loss function $\ell$ is the cross-entropy loss in our experiments.
As for Algorithm~\ref{alg:est-shifts}, the importance sampling size $M = 10000$, and we empirically set the thresholds $\epsilon_\textnormal{div} = 1 \times 10^{-12}$ and $\epsilon_\textnormal{cor} = 5 \times 10^{-4}$. We use Gaussian kernels for all the KDEs.

\section{Discussion on the Convergence of the Hidden Feature}
\label{app:convergence}
In this section, we investigate the convergence of the features extracted by the neural network. We base our analysis on the neural tangent kernel (NTK)~\cite{jacot2018neural, zandieh2021scaling}.
We focus on the dynamic of the output of the feature extractor instead of the output of the entire neural network.
For simplicity, consider a fully-connected neural network with layers numbered from $0$ (input) to $L$ (output), each containing $n_0,\cdots, n_{L-1}$, and $n_L=1$ neurons.
The network uses a Lipschitz, twice-differentiable nonlinearity function $\sigma : \mathbb{R}\rightarrow\mathbb{R}$ with bounded second derivative.
We define the network function by $f(x;\theta)\definedas h^{(L)}(x;\theta)$, where the function $h^{(\ell)}:\mathbb{R}^{n_0}\rightarrow\mathbb{R}^{n_\ell}$ and function $g^{(\ell)}:\mathbb{R}^{n_0}\rightarrow\mathbb{R}^{n_\ell}$ are defined from the $0$-th layer to the $L$-th layer recursively by
\begin{align*}
	g^{(0)}(x;\theta) &\definedas x,\\
	h^{(\ell+1)}(x;\theta) &\definedas \frac{1}{\sqrt{n_\ell}}W^{(\ell)}g^{(\ell)}(x;\theta),\\
	g^{(\ell)}(x;\theta) &\definedas \sigma(h^{(\ell)}(x;\theta)).
\end{align*}
We refer to the output of the ($L-1$)-th layer as the extracted feature, \ie $h^{(L-1)}(x;\theta)$.\footnote{The output of any of the intermediate layer can be regarded as the extracted feature. Our analysis can be easily extend to other scenarios.} Given a training dataset $\{(x_i,y_i)\}_{i=1}^{n}\subset \mathbb{R}^{n_0}\times\mathbb{R}$, consider training the neural network by minimizing the loss function over
training data by gradient descent: $\sum_{i=1}^{n} loss(f(x_i;\theta),y_i)$.
\begin{restatable}[]{theorem}{convergenceoffeature}
   Assume that the non-linear activation $\sigma$ is Lipschitz continuous, twice differentiable with bounded second order derivative. As the width of the hidden layers increase to infinity, sequentially, the hidden layer output $h^{(L-1)}(x;\theta)$ converges to the solution of the differential equation
   \begin{align*}
       \frac{du^{(L-1,1)}(t)}{dt} &= -H^{(L-1,1)}G^{(L-1,1)}(t),\\
	&\vdots\\
	\frac{du^{(L-1,n_{L-1})}(t)}{dt} &= -H^{(L-1,n_{L-1})}G^{(L-1,n_{L-1})}(t).
   \end{align*}
   where $u^{(\cdot,\cdot)}(t)$ is the vectorized form of $h^{(L-1)}(x;\theta)$, $H^{(\cdot,\cdot)}$ is the neural tangent kernel corresponding to hidden layer output, and $G^{(\cdot,\cdot)}(t)$ is the vectorized form of the derivative of the loss function corresponding to the hidden later output.
\end{restatable}
\begin{proof}
    Let $\theta^{(\ell)}$ denote the parameters in the first $\ell$ layers. Then the parameters $\theta^{(\ell)}$ evolve according to the differential equation
    \begin{align*}
    	\frac{d\theta^{(\ell)}(t)}{dt} &= -\nabla_{\theta^{(\ell)}(t)} loss(f(x_i;\theta(t)),y_i)\\
    	&= -\sum_{i=1}^{n}\left[\frac{\partial h^{(\ell)}(x_i;\theta(t))}{\theta^{(\ell)}(t)}\right]^T \nabla_{h^{(\ell)}(x_i;\theta(t))} loss(f(x_i;\theta(t)),y_i),
    \end{align*}
    where $t\ge0$ is the continuous time index, which is commonly used in the analysis of the gradient descent with infinitesimal learning rate. The evolution of the feature $h^{(L-1)}(x_j;\theta)$ of the input $x_j$, $j\in [n]$ can be written as
    \small
    \begin{equation*}
    	\frac{dh^{(L-1)}(x_j;\theta(t))}{dt} = -\sum_{i=1}^{n}\frac{\partial h^{(L-1)}(x_j;\theta(t))}{\theta^{(L-1)}(t)}\left[\frac{\partial h^{(L-1)}(x_i;\theta(t))}{\theta^{(L-1)}(t)}\right]^T \nabla_{h^{(L-1)}(x_i;\theta(t))} loss(f(x_i;\theta(t)),y_i).
    \end{equation*}
    \normalsize
    Let $G^{(L-1,k)}(t) = (\nabla_{h^{(L-1,k)}(x_j;\theta(t))} loss(f(x_j;\theta(t)),y_j))_{j\in[n]}$ denote the gradient of the loss function corresponding to the $k$-th output of the ($L-1$)-th intermediate layer at time $t$, 
    and $u^{(L-1,k)}(t) = (h^{(L-1,k)}(x_j;\theta(t)))_{j\in[n]}$ denote the $k$-th output of the $\ell$-th intermediate layer at time $t$, respectively. The evolution of the feature can be written more compactly as
    \begin{equation*}
    	\frac{du^{(L-1,k)}(t)}{dt} = -H^{(L-1,k)}(t)G^{(L-1,k)}(t),
    \end{equation*}
    where $H^{(L-1,k)}(t)$ is the matrix defined as
    \begin{align*}
    	[H^{(L-1,k)}(t)]_{a,b}
    	&=\left\langle\frac{\partial h^{(L-1,k)}(x_a;\theta(t))}{\partial\theta^{(L-1)}(t)},\frac{\partial h^{(L-1,k)}(x_b;\theta(t))}{\partial\theta^{(L-1)}(t)}\right\rangle.
    \end{align*}
    Applying Theorem 1 and Theorem 2 in \cite{jacot2018neural}, as the width of the hidden layers $n_1,\cdots, n_{L-1} \rightarrow \infty$, sequentially, we have $H^{(L-1,k)}(t)$ converges to a fixed kernel $H^{(L-1,k)}$. Thus, the extracted feature (\ie the output of the ($L-1$)-th layer) converges to the solution of the system of the differential equations below
    \begin{align*}
    	\frac{du^{(L-1,1)}(t)}{dt} &= -H^{(L-1,1)}G^{(L-1,1)}(t),\\
    	&\vdots\\
    	\frac{du^{(L-1,n_{L-1})}(t)}{dt} &= -H^{(L-1,n_{L-1})}G^{(L-1,n_{L-1})}(t).
    \end{align*}
\end{proof}

\section{Effects of the Neural Network Architecture}
\label{app:arch}
We have tested neural network architecture (MLP) with different number of parameters on Colored MNIST, as shown in \Cref{table:arch-cmnist} (estimated values), \Cref{table:cmnist-ttest} (t-test) and \Cref{table:epochs} (different number of training epochs). The results demonstrate no significant difference in the effect of various neural network architectures on the estimation of diversity and correlation shift.
The results tend to converge as model capacity increases. We also compared with other types of architectures (\eg, EfficientNet) on more complicated datasets, PACS, in \Cref{table:arch-pacs}, where we have observed similar trends.
The network architecture can have effects on the estimation, but we expect the architecture to have enough expressive power.

\begin{table}[H]
\small
\centering
\begin{threeparttable}
\begin{tabular}{ccc|ccccc}
\toprule
Network &\# Params &Type &Split 0 &Split 1 &Split 2 &Split 3 &Split 4 \\
\midrule
MLP (dim=16)     &0.0067M & $D_\textnormal{div}$ &0.0001 &0.0001 &0.0002 &0.0001 &0.0001 \\
~   &~   & $D_\textnormal{cor}$ &0.8280 &0.7838 &0.7752 &0.8004 &0.7727 \\
MLP (dim=32)     &0.0140M & $D_\textnormal{div}$ &0.0001 &0.0002 &0.0002 &0.0002 &0.0000 \\
~   &~   & $D_\textnormal{cor}$ &0.7067 &0.7055 &0.7011 &0.4299 &0.6807 \\
MLP (dim=64)     &0.0299M & $D_\textnormal{div}$ &0.0000 &0.0000 &0.0001 &0.0000 &0.0000 \\
~   &~   & $D_\textnormal{cor}$ &0.6901 &0.7279 &0.7096 &0.7366 &0.7195 \\
MLP (dim=200)    &0.1205M & $D_\textnormal{div}$ &0.0000 &0.0003 &0.0000 &0.0000 &0.0000 \\
~   &~   & $D_\textnormal{cor}$ &0.6780 &0.5684 &0.6628 &0.6679 &0.6305\\
MLP (dim=390)    &0.3089M & $D_\textnormal{div}$ &0.0000 &0.0000 &0.0000 &0.0000 &0.0000 \\
~   &~   & $D_\textnormal{cor}$ &0.6971 &0.6788 &0.6680 &0.7343 &0.6846 \\
MLP (dim=1024)    &1.4603M & $D_\textnormal{div}$ &0.0000 &0.0000 &0.0000 &0.0000 &0.0000 \\
~   &~  & $D_\textnormal{cor}$ &0.6609 &0.6639 &0.6701 &0.6765 &0.6248 \\
ResNet-18     &11.1724M & $D_\textnormal{div}$ &0.0000 &0.0000 &0.0000 &0.0000 &0.0001 \\
~   &~    & $D_\textnormal{cor}$ &0.6677 &0.6833 &0.6802 &0.5834 &0.5809 \\
\bottomrule
\end{tabular}
\end{threeparttable}
\caption{Estimated diversity and correlation shift of Colored MNIST on networks with different capacity.}
\label{table:arch-cmnist}
\end{table}

\begin{table}[H]
\small
\centering
\begin{threeparttable}
\begin{tabular}{cc|ccccccc}
\toprule
Network &\makecell{\# Params} &\makecell{MLP\\(dim=16)} &\makecell{MLP\\(dim=32)} &\makecell{MLP\\(dim=64)} &\makecell{MLP\\(dim=200)} &\makecell{MLP\\(dim=390)} &\makecell{MLP\\(dim=1024)} &ResNet-18 \\
\midrule
MLP (dim=16)     &0.0067M &1.0000 &0.6740 &0.8348 &0.7189 &0.7799 &0.7035 & 0.6581 \\
MLP (dim=32)     &0.0140M &0.6740 &1.0000 &0.8277 &0.9480 &0.8832 &0.9639 &0.9852 \\
MLP (dim=64)     &0.0299M &0.8348 &0.8277 &1.0000 &0.8777 &0.9428 &0.8614 &0.8116 \\
MLP (dim=200)    &0.1205M &0.7189 &0.9480 &0.8777 &1.0000 &0.9343 &0.9838 &0.9325 \\
MLP (dim=390)    &0.3089M &0.7799 &0.8832 &0.9428 &0.9343 &1.0000 &0.9180 &0.8672 \\
MLP (dim=1024)   &1.4603M &0.7035 &0.9639 &0.8614 &0.9838 &0.9180 &1.0000 &0.9485 \\
ResNet-18        &11.1724M &0.6581 &0.9852 &0.8116 &0.9325 &0.8672 &0.9485 &1.0000 \\
\bottomrule
\end{tabular}
\caption{T-test on the estimated values between different architectures on Colored MNIST.}
\label{table:cmnist-ttest}
\end{threeparttable}
\end{table}

\begin{table}[H]
\small
\centering
\begin{threeparttable}
\begin{tabular}{ccc|ccccc}
\toprule
Network &\# Epochs &Type &Split 0 &Split 1 &Split 2 &Split 3 &Split 4 \\
\midrule
\makecell{EfficientNet-b0 (4.67M)}    &500 &\makecell{ $D_\textnormal{div}$\\ $D_\textnormal{cor}$} &\makecell{0.0000\\0.7667} &\makecell{0.0000\\0.5177} &\makecell{0.0000\\0.6465} &\makecell{0.0000\\0.1567} &\makecell{0.0000\\0.1690} \\
\makecell{EfficientNet-b0 (4.67M)}    &1000 &\makecell{ $D_\textnormal{div}$\\ $D_\textnormal{cor}$} &\makecell{0.0045\\0.9172} &\makecell{0.0038\\0.8682} &\makecell{0.0016\\0.8099} &\makecell{0.0036\\0.9293} &\makecell{0.0047\\0.9158} \\
\makecell{EfficientNet-b0 (4.67M)}    &2000 &\makecell{ $D_\textnormal{div}$\\ $D_\textnormal{cor}$} &\makecell{0.0028\\0.8350} &\makecell{0.0023\\0.7874} &\makecell{0.0016\\0.8428} &\makecell{0.0022\\0.8871} &\makecell{0.0020\\0.8819} \\
\makecell{EfficientNet-b0 (4.67M)}    &4000 &\makecell{ $D_\textnormal{div}$\\ $D_\textnormal{cor}$} &\makecell{0.0007\\0.9248} &\makecell{0.0010\\0.9204} &\makecell{0.0018\\0.8211} &\makecell{0.0005\\0.8516} &\makecell{0.0020\\0.9763} \\
\bottomrule
\end{tabular}
\end{threeparttable}
\caption{Different number of training epochs on Colored MNIST.}
\label{table:epochs}
\end{table}

\begin{table}[H]
\centering
\begin{adjustbox}{width=\linewidth}
\begin{threeparttable}
\begin{tabular}{ccc|cccc|cccc|cccc|cccc|cccc}
\toprule
Network &\# Params &Type & \multicolumn{4}{c}{Split 0} & \multicolumn{4}{c}{Split 1} & \multicolumn{4}{c}{Split 2} & \multicolumn{4}{c}{Split 3} & \multicolumn{4}{c}{Split 4} \\
~& ~& ~ &0 &1 &2 &3  &0 &1 &2 &3 &0 &1 &2 &3 &0 &1 &2 &3 &0 &1 &2 &3 \\
\midrule
ResNet-18     &11.1724 & $D_\textnormal{div}$ &0.9655 &0.7852 &0.8916 &0.9644 &0.8106 &0.9188 &0.8527 &0.7553 &0.7339 &0.9283 &0.8325 &0.5616 &0.9660 &0.8540 &0.7791 &0.7075 &0.8183 &0.6170 &0.8181 &0.8061
\\
~   &~    & $D_\textnormal{cor}$ &0.0000 &0.0001 &0.0026 &0.0000 &0.0016 &0.0011 &0.0024 &0.0000 &0.0015 &0.0000 &0.0003 &0.0000 &0.0008 &0.0000 &0.0013 &0.0000 &0.0000 &0.0000 &0.0004 &0.0000 \\
EfficientNet-b0     &4.6676 & $D_\textnormal{div}$ &0.9169 &0.7012 &0.6586 &0.9446 &0.8742 &0.9897 &0.9708 &1.5537 &1.0244 &0.9112 &0.9031 &0.5640 &0.8347 &0.9876 &1.0300 &0.6176 &0.7596 &1.1299 &0.9945 &1.2978 \\
~   &~  & $D_\textnormal{cor}$ &0.0000 &0.0012 &0.0004 &0.0000 &0.0000 &0.0000 &0.0000 &0.0000 &0.0003 &0.0001 &0.0023 &0.0000 &0.0001 &0.0003 &0.0005 &0.0000 &0.0007 &0.0007 & 0.0003 &0.0000 \\
EfficientNet-b3     &11.4873 & $D_\textnormal{div}$ &0.7642 &0.9262 &0.9982 &1.1570 &0.7940 &0.8942 &0.9789 &0.9755 &0.6751 &1.0670 &1.0302 &0.8345 &1.0010 &0.5790 &0.9349 &0.6927 &0.8667 &1.1565 &0.8798 &2.6791 \\
~   &~  & $D_\textnormal{cor}$ &0.0000 &0.0000 &0.0002 &0.0000 &0.0001 &0.0000 &0.0182 &0.0000 &0.0004 &0.0004 &0.0000 &0.0000 &0.0005 &0.0000 &0.0015 &0.0000 &0.0000 &0.0003 &0.0000 &0.0000 \\
EfficientNet-b5     &29.3940 & $D_\textnormal{div}$ &0.8967 &0.8822 &0.9572 &0.8732 &0.8544 &0.8792 &1.1918 &0.4856 &0.6623 &0.8086 &0.7450 &1.0365 &0.5031 &0.7685 &1.1752 &0.5001 &0.8214 &0.5249 &1.0050 &0.6596 \\
~  &~  & $D_\textnormal{cor}$ &0.0100 &0.0001 &0.0000 &0.0000 &0.0002 &0.0015 &0.0001 &0.0000 &0.0013 &0.0003 &0.0041 &0.0000 &0.0048 &0.0046 &0.0047 &0.0000 &0.0011 &0.0011 &0.0025 &0.0\\
\bottomrule
\end{tabular}
\end{threeparttable}
\end{adjustbox}
\caption{Different architectures on PACS.}
\label{table:arch-pacs}
\end{table}

\section{Estimation Results with Error Bars}
\label{app:estimation-results}
The table below lists all the results that have been plotted in \Cref{fig:estimation} with standard error bars. The statistics are averaged over five runs of different weight initializations and training/validation splits.
\begin{table}[H]
    \small
    \centering
    \begin{tabular}{lcc}
        \toprule
        Dataset         & Div. shift      & Cor. shift      \\
        \midrule
        i.i.d\onedot data     & 0.00 $\pm$ 0.00 & 0.00 $\pm$ 0.00 \\
        PACS            & 0.81 $\pm$ 0.05 & 0.00 $\pm$ 0.00 \\
        Office-Home     & 0.17 $\pm$ 0.02 & 0.00 $\pm$ 0.00 \\
        Terra Incognita & 0.92 $\pm$ 0.06 & 0.00 $\pm$ 0.00 \\
        Camelyon        & 1.07 $\pm$ 0.80 & 0.00 $\pm$ 0.00 \\
        DomainNet       & 0.43 $\pm$ 0.03 & 0.08 $\pm$ 0.00 \\
        Colored MNIST   & 0.00 $\pm$ 0.00 & 0.55 $\pm$ 0.13 \\
        CelebA          & 0.02 $\pm$ 0.02 & 0.29 $\pm$ 0.04 \\
        NICO            & 0.11 $\pm$ 0.06 & 0.24 $\pm$ 0.08 \\
        ImageNet-A      & 0.02 $\pm$ 0.00 & 0.06 $\pm$ 0.05 \\
        ImageNet-R      & 0.06 $\pm$ 0.01 & 0.21 $\pm$ 0.01 \\
        ImageNet-V2     & 0.01 $\pm$ 0.01 & 0.49 $\pm$ 0.10 \\
        \bottomrule
    \end{tabular}
    \caption{Estimation of diversity and correlation shift.}
    \label{tab:numeric-div-cor}
\end{table}

\section{Datasets}
\label{app:datasets}

Our benchmark includes the following datasets dominated by diversity shift:
\begin{itemize}
    \item \textbf{PACS}~\cite{li2017deeper} is a common DG benchmark. The datasets contain images of objects and creatures depicted in different styles, which are grouped into four domains, $\{ \textnormal{photos}, \textnormal{art}, \textnormal{cartoons}, \textnormal{sketches} \}$. In total, it consists of $9,991$ examples of dimension $(3, 224, 224)$ and $7$ classes.
    \item \textbf{OfficeHome}~\cite{venkateswara2017deep} is another common DG benchmark similar to PACS.
    It has four domains: $\{ \textnormal{art}, \textnormal{clipart}, \textnormal{product}, \textnormal{real} \}$, containing $15,588$ examples of dimension $(3, 224, 224)$ and $65$ classes.
    \item \textbf{Terra Incognita}~\cite{beery2018recognition} contains photographs of wild animals taken by camera traps at different locations in nature, simulating a real-world scenario for OoD generalization.
    Following DomainBed \cite{gulrajani2021in}, our version of this dataset only utilize four of the camera locations, $\{ \textnormal{L100}, \textnormal{L38}, \textnormal{L43}, \textnormal{L46} \}$, covering $24,788$ examples of dimension $(3, 224, 224)$ and $10$ classes.
    \item \textbf{Camelyon17-WILDS}~\cite{koh2020wilds} is a patch-based variant of the Camelyon17 dataset \cite{bandi2019from} curated by WILDS \cite{koh2020wilds}.
    The dataset contains histopathological image slides collected and processed by different hospitals. Data variation among these hospitals arises from sources like differences in the patient population or in slide staining and image acquisition.
    It contains $455,954$ examples of dimension $(3, 224, 224)$ and $2$ classes collected and processed by $5$ hospitals.
\end{itemize}
On the other hand, these datasets are dominated by correlation shift:
\begin{itemize}
    \item \textbf{Colored MNIST}~\cite{arjovsky2019invariant} is a variant of the MNIST handwritten digit classification dataset \cite{lecun1998gradient}. The digits are colored either red or green in a way that each color is strongly correlated with a class of digits. The correlation is different during training and test time, which leads to spurious correlation.
    Following IRM \cite{arjovsky2019invariant}, this dataset contains $60,000$ examples of dimension $(2, 14, 14)$ and $2$ classes.
    \item \textbf{NICO}~\cite{he2020towards} consists of real-world photos of animals and vehicles captured in a wide range of contexts such as ``in water'', ``on snow'' and ``flying''.
    There are 9 or 10 different contexts for each class of animal and vehicle.
    Our version of this dataset simulates a scenario where animals and vehicles are spuriously correlated with different contexts.
    More specifically, we make use of both classes appeared in four overlapped contexts: ``on snow'', ``in forest'', ``on beach'' and ``on grass'' to construct training and test environments (as in \Cref{tab:nico-split}) that are similar to the setting of Colored MNIST.
    In total, our split consists of $4,080$ examples of dimension $(3, 224, 224)$ and $2$ classes.
    \begin{table}[H]
        \centering
        \small
        \label{tab:nico-split}
        \begin{tabular}{c c c c c c}
            \toprule
            \textbf{Environment} & \textbf{Class} & \textbf{on snow} & \textbf{in forest} & \textbf{on beach} & \textbf{on grass}  \\
            \midrule
            Training 1 & Animal  & 10  & 400 & 10  & 400 \\
                                 & Vehicle & 400 & 10  & 400 & 10  \\
            \midrule
            Training 2 & Animal  & 20  & 390 & 20  & 390 \\
                                 & Vehicle & 390 & 20  & 390 & 20  \\
            \midrule
            Validation & Animal  & 50  & 50  & 50  & 50  \\
                                 & Vehicle & 50  & 50  & 50  & 50  \\
            \midrule
            Test       & Animal  & 90  & 10  & 90  & 10  \\
                                 & Vehicle & 10  & 90  & 10  & 90  \\
            \bottomrule
        \end{tabular}
        \captionsetup{width=.85\linewidth}
        \caption{Environment splits of NICO and the number of examples in each group.}
    \end{table}
    \item \textbf{CelebA}~\cite{liu2015deep} is a large-scale face attributes dataset with more than 200K celebrity images, each with 40 attribute annotations. It has been widely investigated in AI fairness studies~\cite{quadrianto2019discovering, wang2020fair, chuang2021fair} as well as OoD generalization research~\cite{sagawa2019distributionally, pezeshki2020gradient}. Similar to the setting proposed by \cite{sagawa2019distributionally}, our version treats ``hair color'' as the classification target and ``gender'' as the spurious attribute.
    We consider a subset of 27,040 images divided into three environments, simulating the setting of Colored MNIST (where there is large correlation shift). We make full use of the group (blond-hair males) that has the least number of images. See \Cref{tab:celeba-split} for more details regarding the environment splits.
    \begin{table}[H]
        \centering
        \small
        \begin{tabular}{c c c c}
            \toprule
            \textbf{Environment} & \textbf{Class} & \textbf{Male} & \textbf{Female} \\
            \midrule
            Training 1 & blond     & 462    & 11,671 \\
                       & not blond & 11,671 & 462    \\
            \midrule
            Training 2 & blond     & 924    & 11,209 \\
                       & not blond & 11,209 & 924    \\
            \midrule
            Test       & blond     & 362  & 362  \\
                       & not blond & 362  & 362  \\
            \bottomrule
        \end{tabular}
        \captionsetup{width=.85\linewidth}
        \caption{Environment splits of CelebA and the number of examples in each group.}
        \label{tab:celeba-split}
    \end{table}
\end{itemize}

\section{Model Selection Methods}
\label{app:model-selection}

Among the three commonly-used model selection methods we used, \emph{training-domain validation} and \emph{test-domain validation} are concisely described in \cite{gulrajani2021in} as follows:
\begin{itemize}
    \item \textbf{Training-domain validation.}
        We split each training domain into training and validation subsets. We train models using the training subsets, and choose the model maximizing the accuracy on the union of validation subsets. This strategy assumes that the training and test examples follow a similar distribution.
    \item \textbf{Test-domain validation.}
        We choose the model maximizing the accuracy on a validation set that follows the distribution of the test domain. We allow one query (the last checkpoint) per choice of hyperparameters, disallowing early stopping.
\end{itemize}

We use \emph{OoD validation} in place of \emph{leave-one-domain-out validation} (another method employed by \cite{gulrajani2021in}) out of two considerations: (i) the Camelyon17 dataset is an official benchmark listed in WILDS \cite{koh2020wilds}, which inherently comes with an OoD validation set; (ii) given $k$ training domains, leave-one-domain-out validation is computationally costly (especially when $k$ is large), increasing the number of experiments by $k-1$ times. Moreover, when $k$ is small (\eg $k=2$ in Colored MNIST), the leave-one-domain-out validation method heavily reduces the number of training examples accessible to the models.
\begin{itemize}
    \item \textbf{OoD validation.}
        We choose the model maximizing the accuracy on a validation set that follows \emph{neither} the distribution of the training domain or the test domain. This strategy assumes that the models generalizing well on the OoD validation set also generalize well on the test set.
\end{itemize}

\section{ImageNet-V2 Experiment}
\label{app:imagenetv2-exp}
Here we include the experiment results on ImageNet-V2 as an example for datasets exhibiting real-world distribution shift.
Experimental comparisons and discussions on ImageNet and its variants are not included in the main paper along with other datasets for three main reasons: \textbf{(i)} ImageNet is seldom considered in DG literature; \textbf{(ii)} the ImageNet variants are released as validation sets and the data size is relatively small (e.g., 10 images per class for ImageNet-V2); and \textbf{(iii)} most of the OoD generalization algorithms assume multiple training domains, however, there is no standard way to construct a multi-domain ImageNet.
Hence, we conduct experiments with the algorithms that do not make the multi-domain assumption on ImageNet (as training domain) and ImageNet-V2 (as test domain).
These algorithms are CORAL, SagNet, and RSC.
As previously shown in \cref{tab:benchmark-div}, they are superior or equivalent to ERM in terms of performance on datasets dominated by diversity shift.
In comparison, the dominant shift between ImageNet and ImageNet-V2 is the correlation shift.
The experiment result on ImageNet-V2 is shown in the table below.
As expected, ERM outperforms the other algorithms.
Note that the relative ranking of the algorithms is reversed from that in \cref{tab:benchmark-div}.
\begin{table}[h]
    \small
    \centering
    \begin{tabular}{cccc}
    \hline
    ERM & CORAL & SagNet & RSC \\
    \hline
    32.4 $\pm$ 0.2 & 31.9 $\pm$ 0.5 & 31.0 $\pm$ 0.4 & 28.3 $\pm$ 1.4\\
    \hline
    \end{tabular}
    \caption{Performance of ERM and OoD generalization algorithms on ImageNet-V2.}
    \label{tab:imagenetv2-exp}
\end{table}

\newpage
\section{Hyperparameter Search Space}
\label{app:hparams-search}

For convenient comparison, we follow the search space proposed in \cite{gulrajani2021in} whenever applicable.

\begin{table}[htbp]
\centering
\scalebox{0.9}{
\begin{tabular}{llll}
\toprule
\textbf{Condition} & \textbf{Hyperparameter} & \textbf{Default value} & \textbf{Random distribution} \\
\midrule
ResNet & learning rate & 0.00005 & $10^{\textnormal{Uniform}(-5,-3.5)}$ \\
       & batch size & 32 & $2^{\textnormal{Uniform}(3,5.5)}$ \\
       & batch size (if CelebA) & 48 & $2^{\textnormal{Uniform}(4.5,6)}$ \\
       & batch size (if ARM) & 8 & 8 \\
       & ResNet dropout & 0 & 0 \\
       & generator learning rate & 0.00005 & $10^{\textnormal{Uniform}(-5,-3.5)}$ \\
       & discriminator learning rate & 0.00005 & $10^{\textnormal{Uniform}(-5,-3.5)}$ \\
       & weight decay & 0 & $10^{\textnormal{Uniform}(-6,-2)}$ \\
       & generator weight decay & 0 & $10^{\textnormal{Uniform}(-6,-2)}$ \\
\midrule
MLP    & learning rate & 0.001 & $10^{\textnormal{Uniform}(-4.5,-3.5)}$ \\
       & batch size & 64 & $2^{\textnormal{Uniform}(3,9)}$ \\
       & generator learning rate & 0.001 & $10^{\textnormal{Uniform}(-4.5,-2.5)}$ \\
       & discriminator learning rate & 0.001 & $10^{\textnormal{Uniform}(-4.5,-2.5)}$ \\
       & weight decay & 0 & 0 \\
       & generator weight decay & 0 & 0 \\
\midrule
IRM	& lambda & 100 &$10^{\textnormal{Uniform}(-1,5)}$	 \\
&	iterations annealing & 500&$10^{\textnormal{Uniform}(0,4)}$\\
&	iterations annealing (if CelebA) & 500&$10^{\textnormal{Uniform}(0,3.5)}$\\
\midrule
VREx	& lambda & 10 &$10^{\textnormal{Uniform}(-1,5)}$	 \\
&	iterations annealing & 500&$10^{\textnormal{Uniform}(0,4)}$\\
&	iterations annealing (if CelebA) & 500&$10^{\textnormal{Uniform}(0,3.5)}$\\
\midrule
Mixup&alpha & 0.2 &$10^{\textnormal{Uniform}(0,4)}$ \\
\midrule
GroupDRO& eta	 &	0.01& $10^{\textnormal{Uniform}(-1,1)}$\\
\midrule
MMD& gamma	 &	1& $10^{\textnormal{Uniform}(-1,1)}$\\
\midrule
CORAL& gamma	 &	1& $10^{\textnormal{Uniform}(-1,1)}$\\
\midrule
MTL& ema	 &	0.99& RandomChoice($[0.5,0.9,0.99,1]$)\\
\midrule
DANN& lambda	 &	1.0& $10^{\textnormal{Uniform}(-2,2)}$\\
 & disc weight decay	 &	0& $10^{\textnormal{Uniform}(-6,2)}$\\
  & discriminator steps	 &	1& $2^{\textnormal{Uniform}(0,3)}$\\
 & gradient penalty	 &	 0& $10^{\textnormal{Uniform}(-2,1)}$\\
 & Adam $\beta_1$	 &	0.5& RandomChoice($[0,0.5]$)\\
\midrule
MLDG & beta & 1 & $10^{\textnormal{Uniform}(-1,1)}$ \\
\midrule
RSC & feature drop percentage & 1/3 & Uniform($0, 0.5$) \\
    & batch drop percentage   & 1/3 & Uniform($0, 0.5$) \\
\midrule
SagNet & adversary weight & 0.1 & $10^{\textnormal{Uniform}(-2,1)}$ \\
\midrule
ANDMask & tau & 1 & Uniform($0.5, 1.0$) \\
\midrule
IGA & penalty & 1,000 & $10^{\textnormal{Uniform}(1,5)}$ \\
\midrule
ERDG & discriminator learning rate & 0.00005  & $10^{\textnormal{Uniform}(-5,-3.5)}$ \\
     & $T^\prime$ learning rate    & 0.000005 & $10^{\textnormal{Uniform}(-6,-4.5)}$ \\
     & $T^{\hphantom{\prime}}$ learning rate & 0.000005 & $10^{\textnormal{Uniform}(-6,-4.5)}$ \\
     & adversarial loss weight   & 0.5 & $10^{\textnormal{Uniform}(-2,0)}$ \\
     & entropy regularization loss weight & 0.01 & $10^{\textnormal{Uniform}(-4,-1)}$ \\
     & cross-entropy loss weight & 0.05 & $10^{\textnormal{Uniform}(-3,-1)}$ \\
\bottomrule
\end{tabular}}
\caption{Hyperparameters, their default values and distributions for random search.}
\label{tab:hparams-search}
\end{table}

\end{document}